\documentclass[11pt,letterpaper]{article}

\usepackage[top=0.7in,bottom=0.7in,left=0.7in,right=0.7in]{geometry}

\usepackage{natbib}

\usepackage[utf8]{inputenc} 
\usepackage[T1]{fontenc}    
\usepackage{hyperref}       
\usepackage{url}            
\usepackage{booktabs}       
\usepackage{amsfonts}       
\usepackage{nicefrac}       
\usepackage{microtype}      

\usepackage[usenames,dvipsnames]{xcolor}   
\usepackage{hyperref}       
\definecolor{darkblue}{rgb}{0.0,0.0,0.65}
\definecolor{darkred}{rgb}{0.68,0.05,0.0}
\definecolor{darkgreen}{rgb}{0.0,0.29,0.29}
\definecolor{darkpurple}{rgb}{0.47,0.09,0.29}
\hypersetup{
   colorlinks = true,
   citecolor  = darkblue,
   linkcolor  = darkred,
   filecolor  = darkblue,
   urlcolor   = darkblue,
 }

\title{Outlier-robust Estimation of a Sparse Linear Model Using Invexity}
\date{}

\author{%
  Adarsh Barik\\
  Department of Computer Science\\
  Purdue University\\
  West Lafayette, IN 47906\\
  \texttt{abarik@purdue.edu} \\
  \and
  Jean Honorio\\
  Department of Computer Science\\
  Purdue University\\
  West Lafayette, IN 47906\\
  \texttt{jhonorio@purdue.edu} \\
}

\usepackage{mathtools} 
\usepackage{booktabs} 
\usepackage{tikz} 
\usetikzlibrary{decorations.pathreplacing,calligraphy}

\usepackage{algorithm}
\usepackage{algorithmic}

\usepackage{subcaption}
\usepackage{amsmath}
\usepackage{amssymb}
\usepackage{mathtools}
\usepackage{amsthm}

\usepackage{arydshln}

\newtheorem{theorem}{Theorem}[section]

\newtheorem{lemma}[theorem]{Lemma}

\theoremstyle{definition}
\newtheorem{definition}[theorem]{Definition}
\newtheorem{assumption}[theorem]{A}
\theoremstyle{remark}

\newtheorem{observation}{Observation}

\newcommand{\calD}{\mathcal{D}}
\newcommand{\calO}{\mathcal{O}}
\newcommand{\calM}{\mathcal{M}}
\newcommand{\calJ}{\mathcal{J}}
\newcommand{\T}{\intercal}

\newcommand{\inner}[2]{\langle #1 \,, #2 \rangle}
\newcommand{\real}{\mathbb{R}}

\newcommand{\E}{\mathbb{E}}
\DeclareMathOperator*{\supp}{Supp}
\newcommand{\uvartheta}{\underline{\widehat{\vartheta}}}
\newcommand{\prob}{\mathbf{P}}
\DeclareMathOperator*{\inertia}{In}
\DeclareMathOperator*{\eig}{eig}

\begin{document}
\maketitle

\begin{abstract}
In this paper, we study problem of estimating a sparse regression vector with correct support in the presence of outlier samples. The inconsistency of lasso-type methods is well known in this scenario. We propose a combinatorial version of outlier-robust lasso which also identifies clean samples. Subsequently, we use these clean samples to make a good estimation. We also provide a novel invex relaxation for the combinatorial problem and provide provable theoretical guarantees for this relaxation. Finally, we conduct experiments to validate our theory and compare our results against standard lasso.
\end{abstract}

\section{Introduction}
\label{sec:introduction}

Many modern day machine learning problems use linear regression as one of their most basic procedures. In high-dimensional setting, a sparse variant of linear regression -- sparse linear regression -- is used to predict the response variable using the high dimensional predictors. In its simplest form, the task is to estimate a sparse regression parameter vector using a given set of predictors and response variables. To ensure sparsity, the most natural approach to solve this problem is to use lasso, i.e, add an $\ell_1$-regularized term to the least square loss minimization. However, in practice, often times we receive samples which do not conform to the linear regression model due to natural or adversarial corruption of both predictors and response~\citep{hampel1986robust}. These samples are treated as outlier samples. Many authors have studied the performance of lasso-type methods in the presence of outlier samples and found that they are not robust~\citep{hadi2009sensitivity}. \cite{fan2001variable,zou2006adaptive} have shown that variable selection for lasso can be inconsistent. While \cite{fan2001variable} proposed a non-convex penalized likelihood based method, \cite{zou2006adaptive} came up with adaptive lasso. Some researchers~\citep{lambert2011robust,sun2020adaptive,dalalyan2019outlier,d2021consistent} have used Huber criterion~\citep{huber2011robust} to deal with outlier samples. 

In this work, we are interested in estimating the sparse regression vector and correctly recovering its support for outlier-robust lasso. These questions have been well studied for lasso and standard sample complexity results are available for it~\citep{wainwright2009sharp,meinshausen2006high}. \cite{chen2013robust}, in their seminal work, have provided some nonconvex trimming methods for this purpose. They also show that convex optimization problems can not be used for outlier-robust support recovery. In this work, we provide an invex formulation for outlier-robust support recovery and provide theoretical guarantees for both estimation of the sparse vector and correct support recovery.           

Our work is inspired by the following observations:

\begin{observation}
	We need \textbf{at least} $\calO(k \log p)$ clean samples to recover a good approximation of the sparse regression vector with exact support where $k$ is the number of non-zero entries in regression vector and $p$ is its dimension~\citep{wainwright2009sharp}. This is also sufficient and necessary number of samples required for correct support recovery in standard lasso. Thus, we cannot hope to do any better than this. 
\end{observation}
\begin{observation}
	Outlier-robust estimation of exact support of sparse regression vector is impossible using convex optimization-based methods. This result follows from~\cite{chen2013robust}.
\end{observation}
\begin{observation}
	If the loss incurred by the outlier sample is less than or equal to the loss incurred by the clean sample, then it is impossible to distinguish between clean and outlier samples. If the clean samples can be identified then they can be used to estimate the regression vector.    
\end{observation}

Based on these observations, we propose a non-convex yet invex relaxation of robust lasso. Invexity keeps it tractable and allows us to provide theoretical guarantees. \cite{lozano2016minimum} have also used invex loss function in their work but their analysis do not rely on invexity of the function. They also do not identify clean samples and provide no guarantees for support recovery. 

\paragraph{Contributions.}
In summary, our contributions can be listed as below:
\begin{enumerate}
\setlength{\itemsep}{0pt}
	\item We provide an invex relaxation for outlier-robust lasso which also identifies clean samples (Section~\ref{sec:invex relaxation}).
	\item We provide theoretical guarantees for a good estimation of true regression vector with exact support recovery using only $\Omega(k^3 \log^2 p)$ clean samples (Theorem~\ref{thm:main theorem}). We employ a primal-dual witness framework for invex problem in our analysis which could be of independent interest. 
	\item Finally, we conduct numerical experiments to validate our theoretical results and compare it against the standard lasso, Adaptive lasso from \cite{lambert2011robust} and Robust lasso from \cite{chen2013robust} (Section~\ref{sec:experiments}). We show that our method outperforms other state-of-the-art methods.  
\end{enumerate}

\section{Data generative process}
\label{sec:data geenration}

In this section, we propose a data generative process for our problem. The data for our problem comes from two models. The first model is called clean model and the other is called outlier model. The clean model uses a true regression parameter vector $\theta^*$ while the outlier model does not conform to a linear model. We use squared loss function $f(X_i, y_i, \theta)$ as sample $(X_i, y_i)$'s measure of fit to a linear model parametrized by $\theta$, i.e., $f(X_i, y_i, \theta) = (y_i - \inner{X_i}{\theta})^2$.
We first describe the clean model below.

\paragraph{Clean model.}  The data generative process for clean model is defined as below:
\begin{align}
	\label{eq:clean model generative process}
	(\forall i \in \{1, \cdots, r\}) \quad y_i = \inner{X_i}{\theta^*} + e_i \; ,
\end{align}   
where $X_i, \theta^* \in \real^p$ and $y_i, e_i \in \real$. Here, $\theta^*$ is a $k$-sparse vector, i.e., it has at most $k$ non-zero entries. More formally, $\theta^*$ is picked from a set $\Theta$ which is defined as: $\Theta = \{ \theta \in \real^p \mid \| \theta \|_0 = k,\; \| \theta \|_1 \leq M \}$ for some $M \geq 0$.
We receive $r$ i.i.d. samples from the clean model and collect them in set $\calM_c$. Without loss of generality, we consider them to be the first $r$ samples. 
\begin{align}
	\label{eq:clean model}
	\calM_c = \{ (X_1, y_1), (X_2, y_2), \cdots, (X_r, y_r) \} \;.
\end{align}
We consider the predictors $X_i \in \real^p, \forall i \in \{ 1, \cdots, r \}$ to be a zero mean sub-Gaussian random vector~\citep{hsu2012tail} with covariance $\Sigma$. In particular, we assume that $\forall j \in \{1, \cdots, p\}, \frac{X_{ij}}{\sqrt{\Sigma_{jj}}}$ is a sub-Gaussian random variable with parameter $\sigma > 0$. The additive noise $e_i$ is sampled independently from a zero mean sub-Gaussian distribution with parameter $\sigma_e$. Next, we characterize our second model which is called $\rho$-outlier model.

\paragraph{$\rho$-Outlier model.} We receive $n - r$ samples from $\rho$-Outlier model which are collected in set $\calM_o$.
\begin{align}
	\label{eq:outlier model}
	\calM_o = \{ (X_{r+1}, y_{r+1}), (X_{r+2}, y_{r+2}), \cdots, (X_n, y_n) \} \;,
\end{align}
where $X_i \in \real^p$ and $y_i \in \real$ for all $i \in \{ r+1, \cdots, n \}$. These outlier samples do not conform to a linear model. However, we do not assume anything about the distribution of samples coming from $\calM_o$. However, we need to characterize outlier property of these samples mathematically. We provide the following necessary criteria for the samples in $\calM_o$. 

\begin{definition}
\label{def:outlier}
All the outlier samples $(X_o, y_o) \in \calM_o$ hold the property that for all $(X_c, y_c) \in \calM_c$, 
\begin{align}
	\label{eq:loss gap}
	f(X_o, y_o, \theta^*) - f(X_c, y_c, \theta^*) \geq \rho \;
\end{align}   
for a given $\rho > 0$.	
\end{definition}

The loss function $f$ provides a measure of fit for a sample to be a clean sample. If equation~\eqref{eq:loss gap} is violated for some samples $(X_o, y_o) \in \calM_o$ and $(X_c, y_c) \in \calM_c$, then an outlier sample from $\calM_o$ will fit the clean model better than a sample from $\calM_c$ which is not desirable. That way definition~\ref{def:outlier} is a natural definition for outlier samples. The following Figure~\ref{fig:row corruption} illustrates our data generative process.

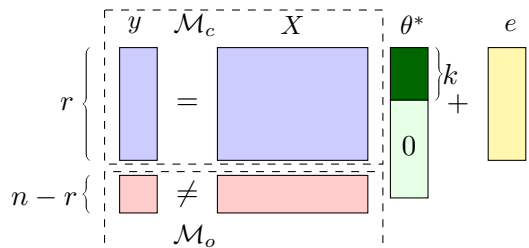
\begin{figure}[!ht]
	\begin{center}
	\begin{tikzpicture}
		\draw[fill=blue!20] (0, 0) rectangle (0.5, -1.5);
		\draw[decorate, decoration = {calligraphic brace, mirror}] (-0.4, 0) -- (-0.4, -1.5);
		\node[] at (-0.7, -0.75) {$r$};
		\draw[fill=red!20] (0, -1.7) rectangle (0.5, -2.2);
		\draw[decorate, decoration = {calligraphic brace, mirror}] (-0.4, -1.7) -- (-0.4, -2.2);
		\node[] at (-1, -2) {$n - r$};
		\draw[fill=blue!20] (1.3, 0) rectangle (3.3, -1.5);
		\draw[fill=red!20] (1.3, -1.7) rectangle (3.3, -2.2);
		\draw[fill=green!40!black] (3.6, 0) rectangle (4.1, -0.7);
		\draw[decorate, decoration = {calligraphic brace}] (4.2, 0) -- (4.2, -0.7);
		\node[] at (4.4, -0.35) {$k$};
		\draw[fill=green!10] (3.6, -0.7) rectangle (4.1, -2);
		\draw[fill=yellow!40] (4.9, 0) rectangle (5.4, -1.5);	
		\node[] at (0.9, -0.75) {$=$};	
		\node[] at (0.9, -1.95) {$\ne$};
		\node[] at (4.5, -0.75) {$+$};
		\node[] at (3.85, -1.3) {$0$};
		\draw[dashed] (-0.2, 0.5) rectangle (3.5, -1.55); 	
		\draw[dashed] (-0.2, -1.65) rectangle (3.5, -2.7); 
		\node[] at (1, 0.3) {\small{$\calM_c$}};
		\node[] at (1, -2.5) {\small{$\calM_o$}};
		\node[] at (0.2, 0.3) {\small{$y$}};
		\node[] at (2.3, 0.3) {\small{$X$}};
		\node[] at (3.9, 0.3) {\small{$\theta^*$}};
		\node[] at (5.2, 0.3) {\small{$e$}};								
	\end{tikzpicture}
	\end{center}
	\caption{\label{fig:row corruption} Illustration of our data generative process}
\end{figure}

Our data generative process falls under the ``row corruption'' model from \cite{chen2013robust}. In that sense, a sample can be made outlier either by corrupting the predictors or the response (or both). Theorem 1 from \cite{chen2013robust} shows that correct support recovery is not possible for our problem using any convex optimization approach. Thus, in the next section, we propose a non-convex yet tractable optimization problem to handle this issue. Our model works on a similar principle as in the Huber loss which is used frequently in outlier-robust lasso models~\citep{lambert2011robust,sun2020adaptive,dalalyan2019outlier,d2021consistent}. Note that Huber loss is defined as
\begin{align}
    \label{eq:huber loss}
    L(X_i, y_i, \theta) = \begin{cases}
        \frac{1}{2} f(X_i, y_i, \theta),\; \text{when $|\sqrt{f(X_i, y_i, \theta)}| \leq \delta$} \\
        \delta (1 - |\sqrt{f(X_i, y_i, \theta)}| - \frac{1}{2} \delta ), \; \text{otherwise}
    \end{cases}\; .
\end{align}
This model reduces sensitivity towards outlier samples. Essentially, it treats a sample $(X_i, y_i)$ as an outlier when  $f(X_i, y_i, \theta) > \delta^2$. When $f(X_i, y_i, \theta) = \delta^2$ 
, the point can be treated either as an outlier or as a clean sample. Our model uses the same principle and classifies a point as an outlier when $f(X_i, y_i, \theta) \geq \delta^2 + \rho$ where $\rho > 0$. The positive $\rho$ ensures that there is no confusion for points which have $f(X_i, y_i, \theta) = \delta^2$. Our results work well for large values of $\rho$. When $\rho$
 is small, then separating the clean samples from the outlier samples becomes difficult.

\section{Invex relaxation}
\label{sec:invex relaxation}

In this section, we propose an invex formulation for outlier-robust estimation of sparse regression parameter vector. Our task is to estimate $\theta^*$ using $(X, y) \in \real^{n \times p + 1}$. The main idea behind our formulation is that we want to use $m$ clean samples to estimate the regression vector. We know that there are $r$ samples from $\calM_c$ and $n - r$ samples from $\calM_o$. We do not need to know the exact value of $r$, but we need that $m \leq r$. We provide a bound on $m$ in our analysis. This leads to a combinatorial problem where we need to choose $m$ out of $n$ available samples. Mathematically, we can formulate the following continuous relaxation of this combinatorial optimization problem:   
\begin{align}
	\label{eq:non-convex optimization problem}
	\begin{matrix}
		\min_{b_1, \cdots, b_n, \theta} & \sum_{i=1}^n b_i f(X_i, y_i, \theta) + \lambda \| \theta \|_1\\
		\text{such that} & 0 \leq b_i \leq 1,\;  i \in \{1, \cdots, n\}, \; \sum_{i=1}^n b_i \geq m  
	\end{matrix} \;,
\end{align}  
where $\lambda \geq 0$ is a regularizer. The $\ell_1$-norm on $\theta$ is to ensure sparsity of the solution. The variable $b_i$ identifies membership of sample $i$ to $\calM_c$. Ideally, we expect exactly $m$ $b_i$'s to take the value $1$ and the remaining $b_i$'s to take the value $0$. Our first result shows that optimization problem~\eqref{eq:non-convex optimization problem} is jointly non-convex with respect to $b_1, \cdots, b_n, \theta$. 

\begin{theorem}
	\label{thm:non-convex}
	The optimization problem~\eqref{eq:non-convex optimization problem} is non-convex with respect to $b_1, \cdots, b_n, \theta$.
\end{theorem} 
\begin{proof}
	\label{proof:non-convex}
	For brevity, we collect all the $b_i$'s in a variable $b$, i.e., $b = (b_1, \cdots, b_n)$. It suffices to show the non-convexity of $g(b, \theta) = \sum_{i=1}^n b_i f(X_i, y_i, \theta)$ because it can be considered as a special case when $\lambda = 0$. Observe that $\frac{\partial g}{\partial b_i} = f(X_i, y_i, \theta), \; i \in \{1,\cdots, n\}, \; 
			\frac{\partial g}{\partial \theta} = -2 \sum_{i=1}^n b_i X_i (y_i - \inner{X_i}{\theta}) $.
	Next, we define the following function $G(b, \bar{b}, \theta, \bar{\theta}) = g(b, \theta) - g(\bar{b}, \bar{\theta}) - \sum_{i=1}^n \frac{\partial g(\bar{b}, \bar{\theta})}{\partial b_i} (b_i - \bar{b}_i) - \inner{\frac{\partial g(\bar{b}, \bar{\theta})}{\partial \theta} }{\theta - \bar{\theta}}$. If we can construct $(b, \theta)$ and $(\bar{b}, \bar{\theta})$ such that $G(b, \bar{b}, \theta, \bar{\theta})$ takes both positive and negative values, then we can show $g(b, \theta)$ to be a non-convex function by definition. To that end, we choose $b = (0, \cdots, 0), \bar{b} = (\frac{1}{2}, \cdots, \frac{1}{2})$, $\theta_i = u$ when $i = t$ and $0$ otherwise, and $\bar{\theta}_i = w$ when $i = t$ and $0$ otherwise.
	With this setting, we get $	G(b, \bar{b}, \theta, \bar{\theta}) = - \sum_{i=1}^n (y_i - X_{it} w) X_{it} (u - w)$ -- if $X_{it}$ happens to be $0$ then we simply change the index $t$. Now, if we fix $u$, we can always choose an appropriate $w$ to make $G(b, \bar{b}, \theta, \bar{\theta})$ take both positive and negative signs. This completes our proof.
\end{proof}

The non-convexity of optimization problem \eqref{eq:non-convex optimization problem} keeps our hopes alive that we can use it to estimate $\theta^*$ with correct support. However, it is also pertinent to discuss limitations of optimization problem \eqref{eq:non-convex optimization problem}. 

\paragraph{Limitations of optimization problem~\eqref{eq:non-convex optimization problem}.} To illustrate the limitations, we provide an adversarial approach to choose $\calM_o$ such that the solution to \eqref{eq:non-convex optimization problem} becomes a bad estimate of $\theta^*$. The adversary chooses an $\widehat{\theta} \in \Theta$ such that $\| \widehat{\theta} \|_1 \leq \| \theta^* \|_1$, $y_o = \inner{X_o}{\widehat{\theta}}, \; \forall (X_o, y_o) \in \calM_o$ and all the outlier samples follow definition~\ref{def:outlier}. The third condition is not that hard to satisfy as clean samples follow equation~\eqref{eq:clean model generative process} and thus, incur small loss with $\theta^*$. A simple example of this scenario would be to consider a mixture of two linear regression models. Now, if $\calM_o$ has at least $m$ samples, then the solution to \eqref{eq:non-convex optimization problem} will be a good estimate of $\widehat{\theta}$ but may not be a good estimate of $\theta^*$. Any outlier-robust model which uses a loss function that depends on $|y_i - \inner{X_i}{\theta}|$ would fail in this case -- this includes all models using Lasso type formulation~\citep{zou2006adaptive} and Huber loss~\citep{lambert2011robust,sun2020adaptive,dalalyan2019outlier,d2021consistent}. However, we can easily fix this problem by converting the problem to a mixed linear regression model~\citep{chen2014convex,yi2014alternating,barik2022sparse}. In fact, on a high level, our approach to outlier-robust lasso problem is similar to~\cite{barik2022sparse} in their use of primal-dual witness techniques~\citep{wainwright2009sharp,wainwright2019high,ravikumar2007spam,ravikumar2010high,daneshmand2014estimating}. However, our formulation is oblivious to the outlier distribution and accommodates extra constraint on number of clean samples in the formulation. This changes the KKT conditions and arguably leads to a more challenging setting where uniqueness of the solution is not guaranteed. In its current form, \eqref{eq:non-convex optimization problem} is difficult to solve. Next, we will provide a relaxation for \eqref{eq:non-convex optimization problem} which has provable theoretical guarantees. In particular, notice that $ f(X_i, y_i, \theta) = \inner{A_i}{\vartheta}$, where $A_i = \begin{bmatrix}
		X_i X_i^T & - X_i y_i \\ - y_i X_i^T & y_i^2
	\end{bmatrix}$ and $\vartheta = \begin{bmatrix}
	\theta \theta^T & \theta \\ 
	\theta^T & 1
\end{bmatrix}$. This leads to the following relaxation for \eqref{eq:non-convex optimization problem}:
\begin{align}
	\label{eq:invex relaxation}
	\begin{matrix}
		\min_{b_1,\cdots,b_n, \vartheta} & \sum_{i=1}^n \inner{b_i A_i}{\vartheta} + \lambda \| \vartheta \|_1 \\
		\text{such that} & \vartheta \succeq 0, \; \vartheta_{p+1, p+1} = 1 \\
		& 0 \leq b_i \leq 1, \; i \in \{1,\cdots,n\}, \; \sum_{i=1}^n b_i \geq m
	\end{matrix} \; ,
\end{align}
where $\| \vartheta \|_1$ is entry-wise $\ell_1$-norm of matrix $\vartheta \in \real^{(p+1) \times (p+1)}$. Next, we show that while optimization problem~\eqref{eq:invex relaxation} is still non-convex, it is tractable. Before that we provide a brief summary of a special class of non-convex functions called ``invex functions''.
\begin{definition}[Invex functions~\citep{hanson1981sufficiency}]
	Let $\phi(t)$ be a function defined on domain $\calD$. Let $\eta$ be a vector valued function defined in $\calD \times \calD$ such that $\inner{\eta(t_1, t_2)}{\nabla \phi(t_2)}$ is well defined for all $t_1, t_2 \in \calD$. Then, $\phi(t)$ is an $\eta$-invex function if 
	\begin{align}
		\label{eq:invexity definition}
		\phi(t_1) - \phi(t_2) \geq \inner{\eta(t_1, t_2)}{\nabla \phi(t_2)},\; \forall t_1, t_2 \in \calD
	\end{align}
\end{definition}

Invex functions can be seen as a generalization of convex function where $\eta(t_1, t_2)$ is simply $t_1 - t_2$. Recently, use of invexity to tackle machine learning problem has gained some traction~\citep{barik2021fair,pinilla2022improved,barik2022sparse}. It is known that KKT conditions are both necessary and sufficient conditions for optimality of invex problems~\citep{hanson1981sufficiency}. Furthermore, \cite{ben1986invexity} has proven that a function is invex if and only if each of its stationary point is a global minimum. We use vector $b \in \real^n$ to denote a vector with i-th entry as $b_i$.  For our problem, we define the domain $\calD$ of the problem as follows: $ \calD =  \{ (b, \vartheta) \mid \vartheta \succeq 0, \; \vartheta_{p+1, p+1} = 1, \; b \in [0, 1]^n, \sum_{i=1}^n b_i \geq m  \} \;$. Next, we show that the objective function of \eqref{eq:invex relaxation} is indeed invex in the domain $\calD$.

\begin{lemma}
	\label{lem:invex objective}
	For all $(b, \vartheta) \in \calD$ and $(\bar{b}, \bar{\vartheta}) \in \calD$, the function $h(b, \vartheta) = \sum_{i=1}^n \inner{b_i A_i}{\vartheta} + \| \vartheta \|_1 $ is invex, i.e., it satisfies the following: $ h(b, \vartheta) - h(\bar{b}, \bar{\vartheta}) - \sum_{i=1}^n \eta_{b_i} \frac{\partial h(\bar{b}, \bar{\vartheta})}{\partial b_i} - \inner{\eta_\vartheta}{\frac{\partial h(\bar{b}, \bar{\vartheta})}{\partial \vartheta}} \geq 0$, where $ 		\eta_{b_i} = \xi_i (b_i - \bar{b}_i), \; \xi_i = \frac{\inner{A_i}{\vartheta}}{\inner{A_i}{\bar{\vartheta}}} \; \forall i = \{1, \cdots, n\}$, and $\eta_\vartheta = \vartheta - \bar{\vartheta}$.
\end{lemma}
\begin{proof}
	\label{proof:invex objective}
	First of all, we observe that $\| \vartheta \|_1$ is a convex function with respect to $\vartheta$ and $\lambda \geq 0$. Thus, it suffices to show that $\bar{h}(b, \vartheta) = \sum_{i=1}^n \inner{b_i A_i}{\vartheta}$ is invex for given $\eta$. Now,
	\begin{align*}
		\frac{\partial \bar{h}(\bar{b}, \bar{\vartheta})}{\partial b_i} = \inner{A_i}{\bar\vartheta}, \; 
		\frac{\partial \bar{h}(\bar{b}, \bar{\vartheta})}{\partial \vartheta} = \sum_{i=1}^n \bar{b}_i A_i	\; .
	\end{align*} 
	By substituting the above values and using simple algebra, it can be shown that
	\begin{align*}
		\bar{h}(b, \vartheta) - \bar{h}(\bar{b}, \bar{\vartheta}) - \sum_{i=1}^n \eta_{b_i} \frac{\partial \bar{h}(\bar{b}, \bar{\vartheta})}{\partial b_i} - \inner{\eta_\vartheta}{\frac{\partial \bar{h}(\bar{b}, \bar{\vartheta})}{\partial \vartheta}} = 0
	\end{align*} 
	When $\frac{\partial \bar{h}(\bar{b}, \bar{\vartheta})}{\partial b_i} = 0$ and $\frac{\partial \bar{h}(\bar{b}, \bar{\vartheta})}{\partial \vartheta} = 0$, we necessarily have $b_i = 0, \; i = \{1, \cdots, n\}$ which is out of the domain $\calD$.  It follows that $h(b, \vartheta)$ is an invex function with respect to the given $\eta$ function. 
\end{proof}

It should be noted that proof of invexity does not readily provide an algorithm to solve problem\eqref{eq:invex relaxation}. In fact, to the best of our knowledge, we are not aware of any general algorithm to solve invex problems with good convergence rate guarantees. In practice, any projected gradient descent type algorithm~\citep{duchi2009efficient} can be used to solve it. In next section, we show that solving optimization problem~\eqref{eq:invex relaxation} leads to a good estimate of $\theta^*$ with correct support.

\section{Estimation and support recovery}
\label{sec:estimation and support recovery}

In this section, we state our main result and provide theoretical analysis. Due to space constraints, we deferred the formal high probability statements of theorems and lemmas (along with their proofs) to the Appendix~\ref{sec:formal statement and proofs} and~\ref{sec:auxiliary lemmas and proofs}.  But first, we discuss some sufficient technical conditions for our analysis. We identify $m$ clean samples from the $r$ clean samples in $X$. However, since ${r \choose m}$ such choices are available, we do not claim uniqueness for our solution to problem~\eqref{eq:invex relaxation}. In particular, we define $\calJ = \{ J \mid J \subseteq \calM_c,\; | J | = m  \}$.
Let $\Sigma$ be the population covariance matrix for clean samples and $\widehat{\Sigma}^J$ be the empirical estimate of $\Sigma$ using clean samples in $J \in \calJ$. Mathematically, $ \Sigma = \E( X_i X_i^T),\forall (X_i, y_i) \in \calM_c,\; \widehat{\Sigma}^J = \frac{\sum_{(X_i, y_i) \in J} X_i X_i^T}{m}$. For clarity, we sometimes drop superscript $J$ from $\widehat{\Sigma}_J$ when the context is clear. Let $\supp(\theta)$ be the support of a vector $\theta$, i.e., $\supp(\theta) = \{ i \in \{1,\cdots,p\} \mid \theta_i \ne 0 \}$. Let $S$ be $\supp(\theta^*)$. We use $\Sigma_{SS}$ to denote $\Sigma$ restricted to rows and columns in $S$. Similarly, $\widehat{\Sigma}^J_{SS}$ denotes $\widehat{\Sigma}^J$ restricted to rows and columns in $\supp(\theta^*)$. We define $\supp^c(\theta)$ as a complement of $\supp(\theta)$, i.e., $\supp^c(\theta) = \{ i \in \{1,\cdots,p\} \mid \theta_i = 0 \}$. We use $S^c = \supp^c(\theta^*)$. We use the following assumptions which are standard in primal-dual witness framework~\citep{wainwright2009sharp,ravikumar2007spam,ravikumar2010high,daneshmand2014estimating}. 
\begin{assumption}
	\label{assum:positive definite Hessian}
	Let $I$ be an identity matrix. Then $ \alpha_1 I \preceq  \Sigma_{SS} \preceq \alpha_2 I$.
\end{assumption} 
\begin{assumption}
	\label{assum:mutual incoherence}
	For some $\kappa \in (0, 1]$, $\| \Sigma_{S^c S} \Sigma_{SS}^{-1} \|_{\infty, \infty} \leq 1 - \kappa  $, where $\|  \|_{\infty, \infty}$ is matrix induced $\infty$-norm.
\end{assumption}
By choosing large enough $m$, it is possible to show that the above assumptions hold in finite sample setting too. However, we defer the proofs to Appendix. Next, we state our main result.

\begin{theorem}[Informal]
	\label{thm:main theorem}
	Let $(\hat{b}, \widehat{\vartheta})$ be the solution to optimization problem~\eqref{eq:invex relaxation}. Under assumptions \ref{assum:positive definite Hessian} and \ref{assum:mutual incoherence}, and after choosing $\lambda \geq \Omega(\sqrt{m \log p})$ and $m \geq \Omega(k^3 \log^2 p)$, the following statements hold true with high probability:
	\begin{enumerate}
\setlength{\itemsep}{0pt}
		\item There exists a $J \in \calJ$ such that $\hat{b}_i = 1$ when $(X_i, y_i) \in J$ and $\hat{b}_i = 0$ otherwise. 
		\item $\widehat{\vartheta}$ is a rank-1 matrix. In particular, $ \widehat{\vartheta} = \begin{bmatrix} \widehat{\theta} \\ 1 \end{bmatrix} \begin{bmatrix} \widehat{\theta}^T & 1 \end{bmatrix}$. 	Moreover,  $\supp^c(\widehat{\theta}) = \supp^c(\theta^*)$ and $\| \widehat{\theta} - \theta^* \|_2 \leq \delta_m$ where $\delta_m \to 0$ as $m \to \infty$.
	\end{enumerate} 
\end{theorem}  
The particular sample complexity of $m$ ensures that assumptions hold in finite sample setting. The choice of $\lambda$ helps to maintain sparsity of $\widehat{\vartheta}$ and makes sure that the entries of $\widehat{\theta}$ which are not in $\supp(\theta^*)$ are zero. By observing that $\| \widehat{\theta} - \theta^* \|_\infty \leq \| \widehat{\theta} - \theta^* \|_2$, one can also recover the correct support of $\theta^*$ (up to sign) as long as $\min_{i \in S} |\theta^*_i| > 2\delta_m$. The proof of Theorem~\ref{thm:main theorem} relies on careful construction of primal variables $(\hat{b}, \widehat{\vartheta})$ and corresponding dual variables and then showing that they satisfy Karush-Kuhn-Tucker (KKT) conditions of optimality with high probability. In remainder of this section, we go through the proof in greater detail. 
We begin our analysis by claiming that $\widehat{\vartheta} \in \real^{(p+1) \times (p+1)}$ is a sparse matrix. In fact, by rearranging the indices of its rows and columns, we claim the following structure for $\widehat{\vartheta}$:
\begin{align}
	\label{eq:sparse vartheta}
	\widehat{\vartheta} = \left[\begin{array}{c;{2pt/2pt}c;{2pt/2pt}c}
		\widehat{\vartheta}_1 & 0 & \widehat{\vartheta}_2 \\ \hdashline[2pt/2pt]
		0 & 0 & 0 \\ \hdashline[2pt/2pt]
		\widehat{\vartheta}_2^T & 0 &  1
	\end{array} \right] \;,
\end{align}
where $\widehat{\vartheta}_1 \in \real^{k \times k}$ collects entries corresponding to rows and columns in $\supp(\theta^*)$, and $\widehat{\vartheta}_2 \in \real^{k \times 1}$ collects entries corresponding to rows in $\supp(\theta^*)$ and $(p+1)$-th column. The zeros in $\widehat{\vartheta}$ are chosen to have appropriate dimensions and correspond to entries related to $\supp^c(\theta^*)$. It remains to prove that $\widehat{\vartheta}$ indeed contains this sparsity structure. We will do this by arguing strict dual feasibility for $\widehat{\vartheta}$ in Section~\ref{subsec:strict dual feasibility}. Let $\underline{\widehat{\vartheta}} \in \real^{k+1 \times k+1}$ be defined as $ \underline{\widehat{\vartheta}} = \begin{bmatrix}
		\widehat{\vartheta}_1  & \widehat{\vartheta}_2 \\
		\widehat{\vartheta}_2^T &  1
	\end{bmatrix}\;. $
By imposing this extra sparsity structure, the optimization problem~\eqref{eq:invex relaxation} can be written compactly as:  
\begin{align}
	\label{eq:compact invex relaxation}
	\begin{matrix}
		\min_{b_1,\cdots,b_n, \widetilde{\vartheta}} & \sum_{i=1}^n \inner{b_i \widetilde{A}_i}{\widetilde{\vartheta}} + \lambda \| \widetilde{\vartheta} \|_1 \\
		\text{such that} & \widetilde{\vartheta} \succeq 0, \; \widetilde{\vartheta}_{k+1, k+1} = 1 \\
		& 0 \leq b_i \leq 1, \; i \in \{1,\cdots,n\}, \; \sum_{i=1}^n b_i \geq m
	\end{matrix} \; ,
\end{align}
where $\widetilde{A}_i$ is a restriction of $A_i$ to rows and columns corresponding to $\widehat{\vartheta}$ in equation~\eqref{eq:sparse vartheta}. Clearly, $(\hat{b}, \underline{\widehat{\vartheta}})$ is the solution to \eqref{eq:compact invex relaxation}. Next, we write necessary and sufficient KKT conditions for \eqref{eq:compact invex relaxation}. 

\subsection{KKT conditions}
\label{subsec:kkt}

We define $\zeta \in \real^{k+1 \times k+1}$ to be a member of subdifferential set of $\| \widetilde{\vartheta} \|_1$ at $\uvartheta$. It follows that $\| \zeta \|_\infty \leq 1$, where $\|\cdot\|_\infty$ is entry-wise infinity norm. Let $(\hat{b}, \underline{\widehat{\vartheta}})$ and $(\Lambda, \mu, \beta, \gamma, \nu)$ be the primal-dual pair (of appropriate dimensions) for equation~\eqref{eq:compact invex relaxation} at optimality. They satisfy the following KKT conditions:
\begin{align}
    &\text{Stationarity condition with respect to $\hat{b}$:} \quad  \inner{\widetilde{A}_i}{\uvartheta} - \beta_i + \gamma_i - \nu = 0\label{eq:stationarity b} \\
    &\text{Stationarity condition with respect to $\uvartheta$:} \quad \sum_{i=1}^n \hat{b}_i \widetilde{A}_i + \lambda \zeta - \Lambda + \mu = 0 \label{eq:stationarity vartheta}\\
    \begin{split}
    &\text{Complementary Slackness conditions:} \quad \inner{\Lambda}{\uvartheta} = 0,\; \nu (m - \sum_{i=1}^n \hat{b}_i) = 0, \; \beta_i \hat{b}_i = 0,\; \\
    &\gamma_i (\hat{b}_i - 1) = 0, \; \forall i \in \{1,\cdots,n\} 		\label{eq:complementary slackness}
    \end{split}\\
    &\text{Dual feasibility conditions:} \quad \Lambda \succeq 0, \; \nu \geq 0, \; \beta_i \geq 0,\; \gamma_i \geq 0, \; \forall i \in \{1,\cdots,n\} \label{eq:dual feasibility}\\
    &\text{Primal feasibility conditions:} \quad \uvartheta \succeq 0,\; \sum_{i=1}^n \hat{b}_i \geq m, \;  0 \leq \hat{b}_i \leq 1, \; i \in \{1,\cdots,n\}\label{eq:primal feasibility}  
\end{align}
where $\mu \in \real^{k+1 \times k+1}$ with only $\mu_{k+1, k+1}$ being the non-zero entry.
We provide a setting for both primal and dual variables which satisfy KKT conditions in our next result. To do so, we define the following optimization problem:
\begin{align}
	\label{eq:underline theta and Jstar}
	\begin{split}
	\underline{\theta}, J^* =& \arg\min_{\widetilde{\theta} \in \real^k,\; J \in \calJ} \sum_{(X_i, y_i) \in J} (y_i - \inner{\widetilde{X}_i}{\widetilde{\theta}})^2 + \lambda (\|\widetilde{\theta}\|_1 + 1)^2 \;.
	\end{split}  
\end{align}
We state the following theorem using the results from \eqref{eq:underline theta and Jstar}.
\begin{theorem}[Informal]
	\label{thm:kkt conditions}
	Under assumptions~\ref{assum:positive definite Hessian} and \ref{assum:mutual incoherence} and after choosing $\lambda \geq \Omega(\sqrt{m \log p})$ and $m \geq \Omega(k^3 \log^2 p)$, the following setting of primal-dual pair $(\hat{b}, \underline{\widehat{\vartheta}})$ and $(\Lambda, \mu, \beta, \gamma, \rho)$ satisfy all the KKT conditions for optimization problem~\eqref{eq:compact invex relaxation} with high probability:
    \begin{align*}
		&\uvartheta = \begin{bmatrix}
			\underline{\theta} \\ 1
		\end{bmatrix} \begin{bmatrix}
		\underline{\theta}^T & 1
	\end{bmatrix} ,\quad \max_{(X_i, y_i) \in J^*} \inner{\widetilde{A}_i}{\uvartheta} \leq \nu \leq \min_{(X_i, y_i) \notin J^*} \inner{\widetilde{A}_i}{\uvartheta}, \\
		&\hat{b}_i = 1,\; \beta_i = 0,\; \gamma_i = \nu - \inner{\widetilde{A}_i}{\uvartheta},  \; (X_i, y_i)\in J^*, \\
  	&\hat{b}_i = 0, \; \gamma_i = 0,\; \beta_i = \inner{\widetilde{A}_i}{\uvartheta} - \nu, \; (X_i, y_i) \notin J^*, \\
		&\Lambda = \sum_{(X_i, y_i) \in J^*} \widetilde{A}_i + \lambda \zeta + \mu, \quad \mu_{k+1, k+1} = - \inner{\sum_{(X_i, y_i) \in J^*} \widetilde{A}_i + \lambda \zeta}{\uvartheta}
 \end{align*}
\end{theorem}

It should be noted that since choice of $J^* \in \calJ$ may not be unique. It can be easily verified that stationarity conditions~\eqref{eq:stationarity b}, \eqref{eq:stationarity vartheta}, complementary slackness conditions~\eqref{eq:complementary slackness} and primal feasibility conditions~\eqref{eq:primal feasibility} are satisfied. However, we will need to verify that setting of variables in Theorem~\ref{thm:kkt conditions} satisfy dual feasibility conditions~\eqref{eq:dual feasibility}.

\paragraph{Non-negativity of $\nu, \beta$ and $\gamma$.} The construction of dual variables ensures that non-negativity of $\nu, \beta$ and $\gamma$ holds as long as the following inequalities are feasible. 
\begin{align}
	\label{eq:feasibility of nu} 
	\max_{(X_i, y_i) \in J^*} \inner{\widetilde{A}_i}{\uvartheta} \leq \nu \leq \min_{(X_i, y_i) \notin J^*} \inner{\widetilde{A}_i}{\uvartheta}
\end{align}
We argue for the feasibility of \eqref{eq:feasibility of nu} in two parts. First, if $\arg\min_{(X_i, y_i) \notin J^*} \inner{\widetilde{A}_i}{\uvartheta} \in \calM_c \setminus J^* $, then by definition of \eqref{eq:underline theta and Jstar} the setting of $\nu$ is feasible -- or else it leads to the contradiction that $J^*$ is the optimal choice for fixed $\underline{\theta}$. However, we still need to show the feasibility of $\nu$ when $\arg\min_{(X_i, y_i) \notin J^*} \inner{\widetilde{A}_i}{\uvartheta} \in \calM_o $. To see this, we prove the following two lemmas. Let $\underline{\theta}^* \in \real^k$ be the restriction of $\theta^*$ on its support $\supp(\theta^*)$. First, we show that $\underline{\theta}$ and $\underline{\theta}^*$ are quite close.

\begin{lemma}[Informal]
	\label{lem:l2 norm close}
	Under assumptions~\ref{assum:positive definite Hessian} and \ref{assum:mutual incoherence} and if we choose $\lambda \geq \Omega(\sqrt{m \log p})$ and $m \geq \Omega(k^3 \log^2 p)$, then $\| \underline{\theta} - \underline{\theta}^* \|_2 \leq \delta_m$
	with high probability. Here, $\delta_m \to 0$ as $m \to \infty$. 
\end{lemma} 
Combining the result from Lemma~\ref{lem:l2 norm close} with the definition~\ref{def:outlier} of outlier samples, we can show that $\max_{(X_i, y_i) \in \calM_c} f(\tilde{X}_i, y_i, \underline{\theta}) \leq \min_{(X_i, y_i) \in \calM_o} f(\tilde{X}_i, y_i, \underline{\theta})$.
\begin{lemma}[Informal]
	\label{lem:nu is feasible}
	Under Assumptions~\ref{assum:positive definite Hessian} and \ref{assum:mutual incoherence} and after choosing $\lambda \geq \Omega(\sqrt{m \log p})$ and $m \geq \Omega(k^3 \log^2 p)$, setting of $\nu$ in \eqref{eq:feasibility of nu} is feasible with high probability for sufficiently large $\rho$.
\end{lemma} 

Now that we have verified non-negativity of dual variables, we show that $\Lambda$ in Theorem~\ref{thm:kkt conditions} is positive semidefinite. We do this by showing two results. First, we establish that $\Lambda$ has a zero eigenvalues using the stationarity KKT condition for \eqref{eq:underline theta and Jstar}.

\begin{lemma}
	\label{lem:zero eigenvalue}
	The dual variable $\Lambda$ from Theorem~\ref{thm:kkt conditions} has zero eigenvalue corresponding to $\begin{bmatrix}
		\underline{\theta} & 1
	\end{bmatrix}^\T$.
\end{lemma}         

Our next result shows that second minimum eigenvalue of $\Lambda$ is strictly greater than zero. We use equation~\eqref{eq:underline theta for kkt} along with Haynesworth's inertia formula~\citep{haynsworth1968determination} to show this result.
\begin{lemma}[Informal]
	\label{lem:second eigenvalue}
	Under assumption \ref{assum:mutual incoherence} and by choosing $m \geq \Omega(k^3 \log^2 p)$, we ensure that second minimum eigenvalue of $\Lambda$ is strictly positive with high probability. 
\end{lemma}  

Lemma~\ref{lem:second eigenvalue} along with Lemma~\ref{lem:zero eigenvalue} ensures that $\Lambda$ satisfies the dual feasibility condition. It also shows that $\Lambda$ has only one vector in its nullspace and thus, setting of $\uvartheta$ is unique in Theorem~\ref{thm:kkt conditions}. Now that we have verified that our setting for  primal and dual variables satisfies all the KKT conditions, we can start assembling all our results to provide a solution for optimization problem~\eqref{eq:invex relaxation}.

\subsection{Assembling solution to invex relaxation}
\label{subsec:solution to invex relaxation}

We can expand the construction of $\uvartheta$ to provide a setting of $\widehat{\vartheta}$ using equation~\eqref{eq:sparse vartheta}:
\begin{align}
	\label{eq:final sparse vartheta}
	\widehat{\vartheta} = \left[\begin{array}{c;{2pt/2pt}c;{2pt/2pt}c}
		\underline{\theta}\underline{\theta}^T & 0 & \underline{\theta} \\ \hdashline[2pt/2pt]
		0 & 0 & 0 \\ \hdashline[2pt/2pt]
		\underline{\theta}^T & 0 &  1
	\end{array} \right] \implies \widehat{\theta} = \begin{bmatrix}
	\underline{\theta} \\ 0
\end{bmatrix} 
\end{align}

The construction of $J^*$ validates the first statement of Theorem~\ref{thm:main theorem} and the construction of $\widehat{\theta}$ from \eqref{eq:final sparse vartheta} immediately shows that the second statement of Theorem~\ref{thm:main theorem} also holds. It only remains to verify strict dual feasibility for $\widehat{\vartheta}$.

\subsection{Strict dual feasibility}
\label{subsec:strict dual feasibility}

It suffices to show that $\supp^c(\widehat{\theta}) = \supp^c(\theta^*)$ to verify strict dual feasibility as we can use the construction of $J^*$, $\widehat{\vartheta}$ and $\widehat{\theta}$ to argue the same for $\widehat{\vartheta}$. We consider the following simplified optimization problem for our arguments:
\begin{align}
	\label{eq:opt problem for strict dual feasibility}
	\widehat{\theta} = \arg\min_{\theta \in \real^p} \sum_{(X_i, y_i) \in J^*} (y_i - \inner{X_i}{\theta}) + \lambda (\| \theta \|_1 + 1)^2
\end{align} 
Let $\omega$ be an element of the subdifferential set of $\| \theta \|_1$ at $\widehat{\theta}$. We denote $\omega$ restricted to $\supp(\theta^*)$ by $\widetilde{\omega}$ and $\omega$ restricted to $\supp^c(\theta^*)$ by $\bar{\omega}$. Our task is to show that $\| \bar{\omega} \|_\infty < 1$. Observe that all the samples in $J^*$ follow data generative process~\eqref{eq:clean model generative process}, i.e, $y_i = \inner{X_i}{\theta^*} + e_i, \forall i \in J^*$. We also note that $\widehat{\theta}$ satisfies the stationarity KKT condition for \eqref{eq:opt problem for strict dual feasibility}, i.e., $ \widehat{\Sigma}^{J^*} (\widehat{\theta} - \theta^*) - \frac{1}{m} \sum_{(X_i, y_i) \in J^*} X_i e_i + \frac{\lambda}{m}(\| \widehat{\theta} \|_1 + 1) \omega = 0 $. We segregate the entries in $\supp(\theta^*)$ and $\supp^c(\theta^*)$ and write:

\begin{align*}
	&\widehat{\Sigma}^{J^*}_{SS} (\widehat{\theta} - \theta^*) - \frac{1}{m} \sum_{(X_i, y_i) \in J^*} \widetilde{X}_i e_i + \frac{\lambda}{m}(\| \widehat{\theta} \|_1 + 1) \widetilde{\omega} = 0 \; \\
	&\widehat{\Sigma}^{J^*}_{S^cS} (\widehat{\theta} - \theta^*) - \frac{1}{m} \sum_{(X_i, y_i) \in J^*} \overline{X}_i e_i + \frac{\lambda}{m}(\| \widehat{\theta} \|_1 + 1) \bar{\omega} = 0 \;,
\end{align*}
where $\overline{X}_i$ denotes $X_i$ restricted to the rows in $\supp^c(\theta^*)$. A little algebraic manipulation gives us the following equation:
\begin{align}
	\label{eq:expression for omega bar}
	\begin{split}
	&\frac{\lambda}{m}(1 + \| \widehat{\theta} \|_1) \bar{\omega} = - \widehat{\Sigma}^{J^*}_{S^cS} \widehat{\Sigma}^{J^* -1}_{SS} \left( \frac{1}{m} \sum_{(X_i, y_i) \in J^*} \widetilde{X}_i e_i  - \frac{\lambda}{m}(1 + \| \widehat{\theta} \|_1) \widetilde{\omega} \right) +  \frac{1}{m} \sum_{(X_i, y_i) \in J^*} \overline{X}_i e_i
	\end{split} 
\end{align}
The next lemma shows that we can bound $\| \bar{\omega} \|_{\infty}$ by using assumptions~\ref{assum:positive definite Hessian}, \ref{assum:mutual incoherence} and bounding term in right hand side of \eqref{eq:expression for omega bar} using concentration inequalities.
\begin{lemma}[Informal]
	\label{lem:bound omegabar}
	Under assumptions~\ref{assum:positive definite Hessian}, \ref{assum:mutual incoherence} and after choosing $\lambda \geq \Omega(\sqrt{m \log p})$ and $m \geq \Omega(k^3 \log^2 p)$, the following bound on $\bar{\omega}$ from \eqref{eq:expression for omega bar} holds with high probability $ \| \bar{\omega} \|_{\infty} \leq 1 - \frac{\kappa}{4}$.
 \end{lemma}
This implies that $\widehat{\theta}$ has the same support as $\theta^*$ which in turn, means that the sparsity structure of $\widehat{\vartheta}$ in \eqref{eq:sparse vartheta} indeed holds with high probability. 

\subsection{Choosing $m$}
Our assumption is that the practitioner needs to know only a lower bound of the number of clean samples, which we believe is a reasonable assumption. We always recommend choosing $m = \Sigma(k^3 \log^2 p)$. There might be three possible cases. First, $r < \calO(k \log p)$ 
 and theoretical guarantees for that regime are impossible even for the case where all data is clean~\citep{wainwright2009sharp}. Second, $r \geq \Omega(k^3 \log^2 p)$ and our method works with high probability. In the third case, when $\Omega(k \log p) < r < \calO(k^3 \log^2 p)$
, our method can still be applied (but we do not provide any guarantees for this regime).

\begin{figure*}[!ht]
	\centering
	\begin{subfigure}{.33\textwidth}
		\centering
		\includegraphics[width=\linewidth]{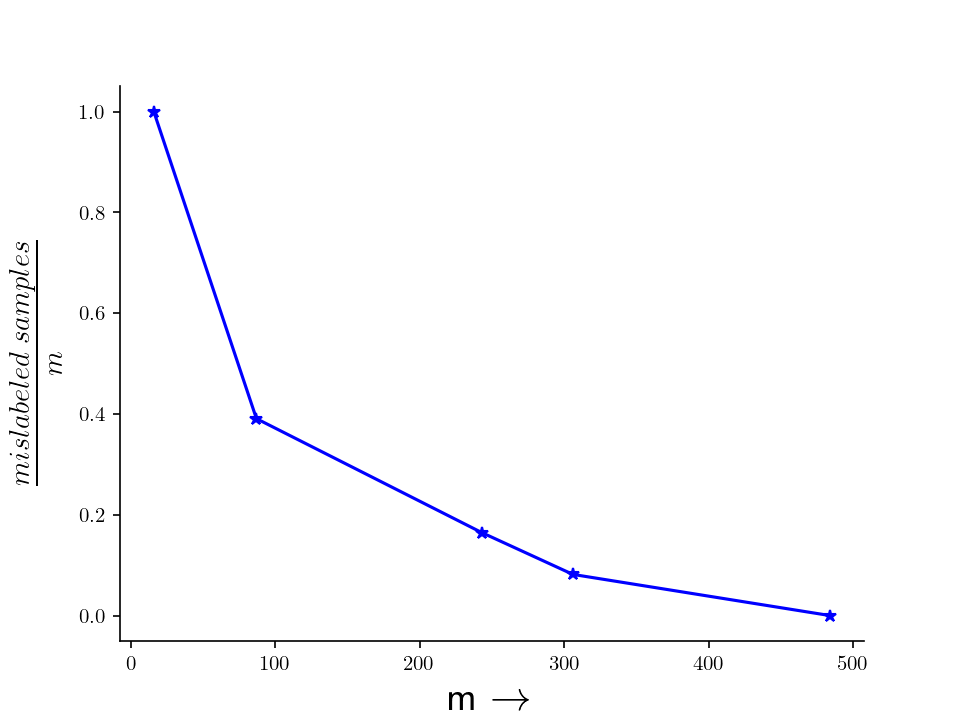}
		\caption{$p=50, k=4$}
		\label{fig:recnumsample}
	\end{subfigure}%
	\begin{subfigure}{.33\textwidth}
		\centering
		\includegraphics[width=\linewidth]{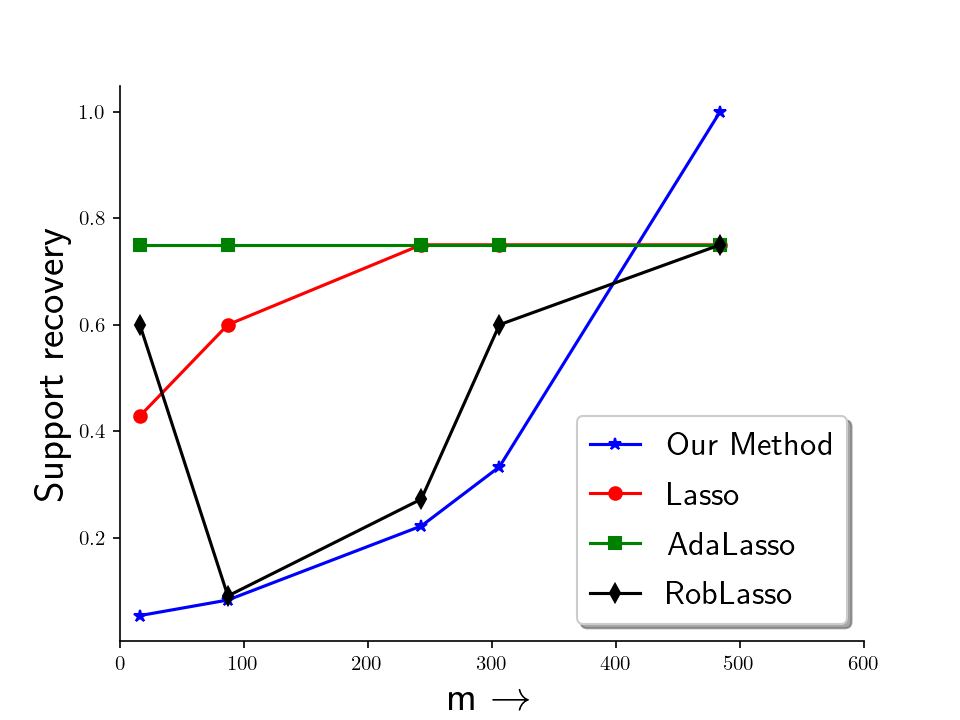}	
		\caption{$p=50, k=4$ }
		\label{fig:recnumsamplecp}
	\end{subfigure}%
	\begin{subfigure}{.33\textwidth}
		\centering
		\includegraphics[width=\linewidth]{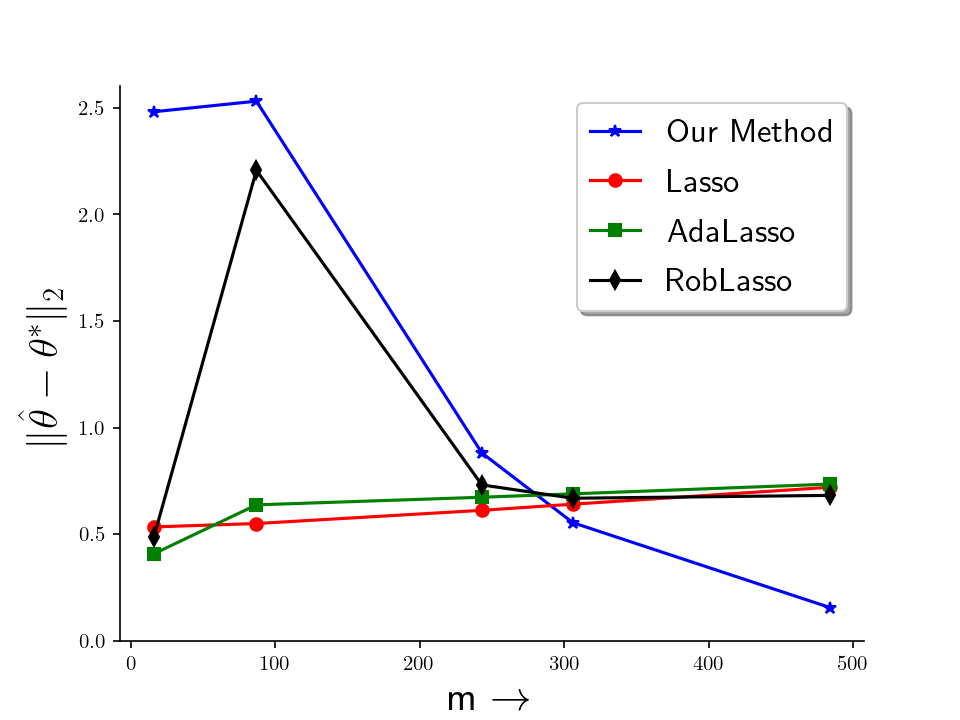}
		\caption{$p=50, k=4$}
		\label{fig:normerror}
	\end{subfigure}
        \begin{subfigure}{.33\textwidth}
		\centering
		\includegraphics[width=\linewidth]{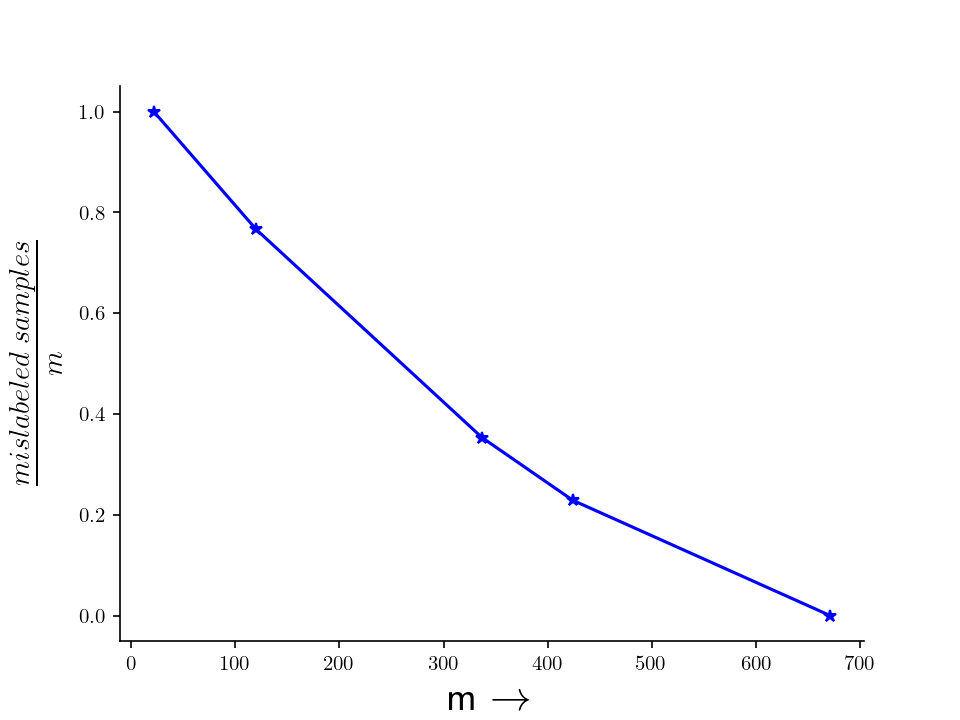}
		\caption{$p=100, k=4$}
		\label{fig:recnumsample1}
	\end{subfigure}%
	\begin{subfigure}{.33\textwidth}
		\centering
		\includegraphics[width=\linewidth]{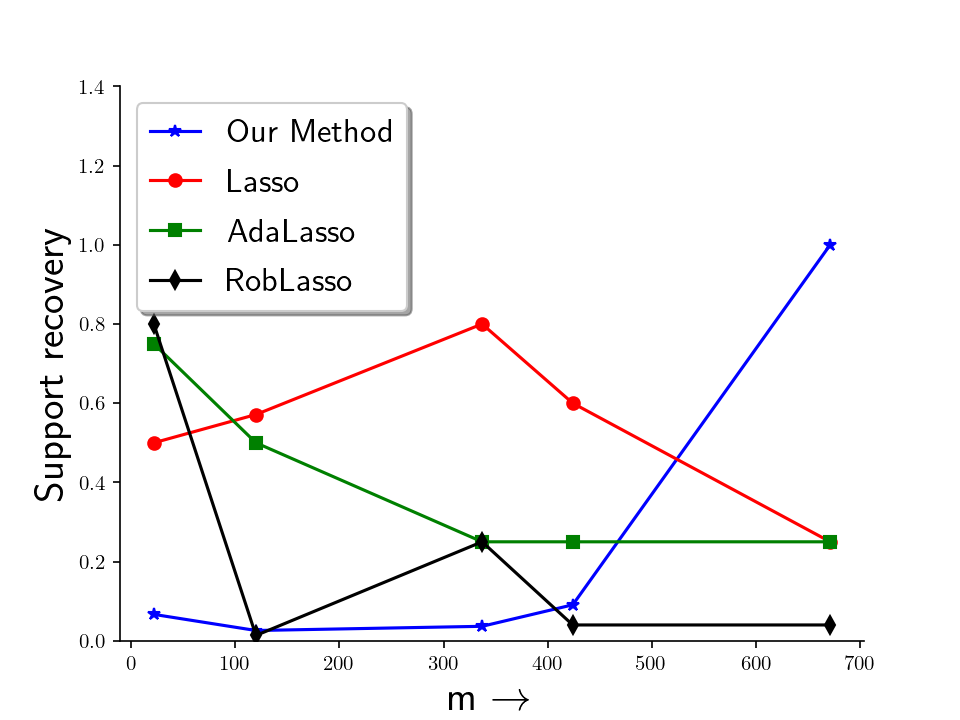}	
		\caption{$p=100, k=4$ }
		\label{fig:recnumsamplecp1}
	\end{subfigure}%
	\begin{subfigure}{.33\textwidth}
		\centering
		\includegraphics[width=\linewidth]{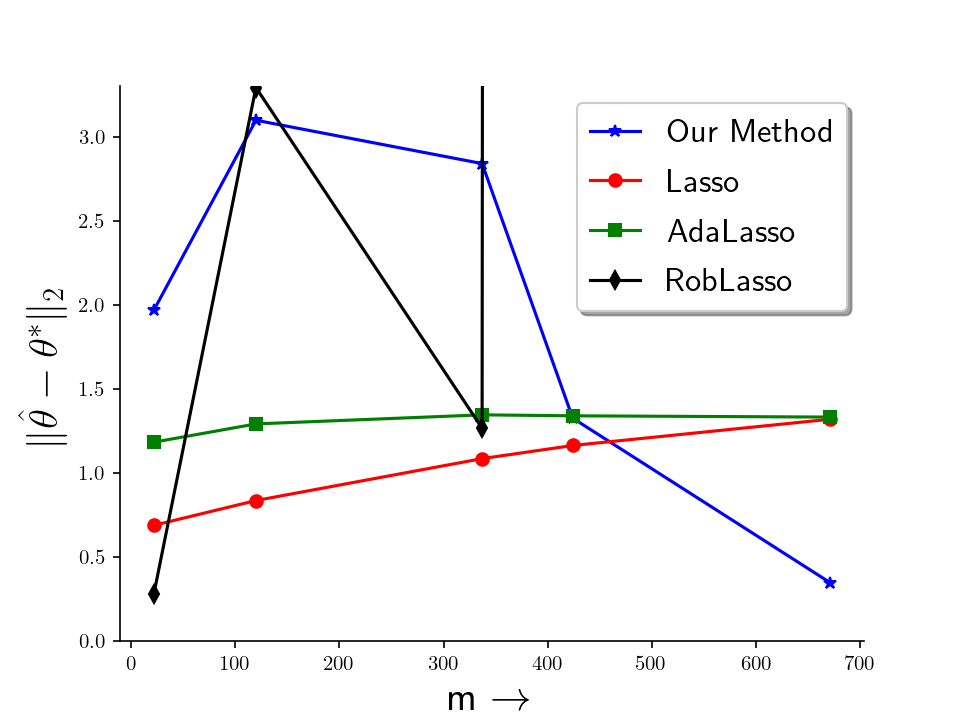}
		\caption{$p=100, k=4$}
		\label{fig:normerror1}
	\end{subfigure}
	\caption{(Left) Mistakes in clean sample recovery as a proportion of $m$ for $p=50, k=4$(top) and $p=100, k=4$(bottom) against $m$. (Middle) Support recovery (intersection/union) for $p=50, k=4$(top) and $p=100, k=4$(bottom) against $m$. (Right) Norm error for $p=50, k=4$(top) and $p=100, k=4$(bottom) against $m$. Comparison to standard lasso, AdaLasso~\citep{lambert2011robust} and RobLasso~\citep{chen2013robust} is also shown.}
	\label{fig:recovery}
\end{figure*}  

\section{Experimental validation}
\label{sec:experiments}

In this section, we validate our theoretical results by running numerical experiments. We also provide comparison of our method with standard lasso, adaptive lasso (AdaLasso) from~\cite{lambert2011robust} and robust lasso (RobLasso) from~\cite{chen2013robust}.
\paragraph{Data generation and setup.} We generate clean samples for our experiments by using zero mean standard normal distribution for the predictors and an independent zero mean additive Gaussian error. The true parameter vector is kept sparse with non-zero entries sampled uniformly at random between $[-1.1, -0.1] \cup [0.1, 1.1]$. The response is generated using equation~\eqref{eq:clean model generative process}. The samples for outlier model are generated by picking predictors uniformly at random from $[0, 1]$ and response from $[0, 5]$. Based on the dimension $p$, number of clean samples are kept to be at $1.1 \times 10^{1.5} \log^2 p$. The outliers are kept at half of this number. This is a difficult scenario where as many as $33\%$ samples are corrupted. According to our theory, we vary $m$ as $10^C \log^2 p$ where $C$ is a control parameter. The regularizer $\lambda$ is chosen according to our theory and follows a order of $\Omega(\sqrt{m \log p})$. It is kept same across all the baseline methods.

\paragraph{Experimental results.} We use projected gradient descent algorithm~\cite{duchi2009efficient} to solve problem~\eqref{eq:invex relaxation}. We first discover the recovery of clean samples -- a guarantee that only our method provides.  In Figure~\ref{fig:recnumsample} and \ref{fig:recnumsample1}, the number of mistakes in recovering clean samples (as a proportion of $m$) goes towards zero for both $p=50$ and $100$ respectively as $m$ increases. This is in-line with our theoretical guarantees. The support recovery is measured by computing $\frac{| \supp(\widehat{\theta}) \cap \supp(\theta^*)|}{|\supp(\widehat{\theta}) \cup \supp(\theta^*)|}$. In Figure~\ref{fig:recnumsamplecp} and~\ref{fig:recnumsamplecp1}, we see that support recovery for our method goes towards $1$ as $m$ increases which validates our theory. The two curves (for $p=50$ and $100$) for our method align quite nicely as is expected from our theory. We see that all the other baseline methods struggle to recover correct support in both the cases. Finally, we compare the norm error $\| \widehat{\theta} - \theta^* \|_2$ to measure the quality of our estimate of $\theta^*$. We observe in Figure~\ref{fig:normerror} and~\ref{fig:normerror1} that the norm error goes towards zero for our method as we increase $m$ for both the scenarios. Other baseline methods again struggle to reduce norm error in presence of outlier samples (RobLasso~\cite{chen2013robust} in fact diverges). These numerical experiments validate our theoretical results and also show that our method outperforms other baseline methods when there are many outliers in the samples.

\paragraph{Concluding remarks.}
In this paper, we have tackled the problem of outlier-robust lasso with a new point of view. We first introduced a combinatorial formulation and then proposed an invex relaxation for the same. The formulation is non-convex but it becomes tractable by using the properties of invex functions. We finally proposed a primal-dual witness framework for our invex relaxation to provide theoretical guarantees. In future it would be interesting to extend our formulation to a setting where $m > r$, i.e., we recover all the clean samples with possibly some outlier samples.

\bibliographystyle{apalike}
\bibliography{huber_lasso}


\newpage

\appendix

\section{Formal Statements and Proofs of Theorems and Lemmas}
\label{sec:formal statement and proofs}

\subsection{Formal Statement of Theorem~\ref{thm:main theorem}}
\label{subsec:formal main theorem}

\paragraph{Theorem~\ref{thm:main theorem}}
\emph{Let $(\hat{b}, \widehat{\vartheta})$ be the solution to optimization problem~\eqref{eq:invex relaxation}. Under assumptions \ref{assum:positive definite Hessian} and \ref{assum:mutual incoherence}, and after choosing $\lambda \geq \frac{64 \sigma \sigma_e}{\kappa} \sqrt{m \log p}$ and $ m \geq \Omega(\frac{k^3 \log^2 p}{\tau_0(\alpha_1, \kappa, \sigma, \Sigma )})$, the following statements hold true with probability at least $1 - \calO(\frac{1}{p})$:
	\begin{enumerate}
		\item There exists a $J \in \calJ$ such that $\hat{b}_i = 1$ when $(X_i, y_i) \in J$ and $\hat{b}_i = 0$ otherwise. 
		\item $\widehat{\vartheta}$ is a rank-1 matrix. In particular, $ \widehat{\vartheta} = \begin{bmatrix} \widehat{\theta} \\ 1 \end{bmatrix} \begin{bmatrix} \widehat{\theta}^T & 1 \end{bmatrix}$. 	Moreover,
		\begin{enumerate}
			\item $\supp^c(\widehat{\theta}) = \supp^c(\theta^*)$.
			\item $\| \widehat{\theta} - \theta^* \|_2 \leq \delta_m$ where $\delta_m  = (2 + M) \frac{2 \lambda \sqrt{k}}{\alpha_1 m}$.
		\end{enumerate} 
	\end{enumerate}
	Here $\tau_0(\alpha_1, \kappa, \sigma, \Sigma )$ is a constant which is independent of $p, k, m$ and $n$. 
}

\subsection{Formal Statement of Theorem~\ref{thm:kkt conditions}}
\label{subsec:formal kkt theorem}
\paragraph{Theorem~\ref{thm:kkt conditions}}
\emph{Under assumptions~\ref{assum:positive definite Hessian} and \ref{assum:mutual incoherence} and after choosing $\lambda \geq \frac{64 \sigma \sigma_e}{\kappa} \sqrt{m \log p}$ and $ m \geq \Omega(\frac{k^3 \log^2 p}{\tau_0(\alpha_1, \kappa, \sigma, \Sigma )})$, the following setting of primal-dual pair $(\hat{b}, \underline{\widehat{\vartheta}})$ and $(\Lambda, \mu, \beta, \gamma, \rho)$ satisfy all the KKT conditions for optimization problem~\eqref{eq:compact invex relaxation} with probability at least $1 - \calO(\frac{1}{p})$:
	\begin{enumerate}
		\item $\uvartheta = \begin{bmatrix}
			\underline{\theta} \\ 1
		\end{bmatrix} \begin{bmatrix}
			\underline{\theta}^T & 1
		\end{bmatrix}$
		\item $\hat{b}_i = 1,\; \beta_i = 0,\; \gamma_i = \nu - \inner{\widetilde{A}_i}{\uvartheta},  \; (X_i, y_i)\in J^*$
		\item $\hat{b}_i = 0, \; \gamma_i = 0,\; \beta_i = \inner{\widetilde{A}_i}{\uvartheta} - \nu, \; (X_i, y_i) \notin J^*$
		\item $\max_{(X_i, y_i) \in J^*} \inner{\widetilde{A}_i}{\uvartheta} \leq \nu \leq \min_{(X_i, y_i) \notin J^*} \inner{\widetilde{A}_i}{\uvartheta} $
		\item $\Lambda = \sum_{(X_i, y_i) \in J^*} \widetilde{A}_i + \lambda \zeta + \mu$
		\item $\mu_{k+1, k+1} = - \inner{\sum_{(X_i, y_i) \in J^*} \widetilde{A}_i + \lambda \zeta}{\uvartheta}$ 
	\end{enumerate}
Here $\tau_0(\alpha_1, \kappa, \sigma, \Sigma )$ is a constant which is independent of $p, k, m$ and $n$.
}

\subsection{Proof of Lemma~\ref{lem:l2 norm close}}
\label{subsec: proof of lemma norm close}

\paragraph{Lemma~\ref{lem:l2 norm close}}
\emph{Under assumptions~\ref{assum:positive definite Hessian} and \ref{assum:mutual incoherence} and if we choose $\lambda \geq 8 \sigma \sigma_e \sqrt{m \log p})$ and $ m \geq \Omega(\frac{k^3 \log^2 p}{\tau_1(\alpha_1, \kappa, \sigma, \Sigma )})$, then with probability at least $1 - \calO(\frac{1}{p})$, $	\| \underline{\theta} - \underline{\theta}^* \|_2 \leq \delta_m $, 
	where $\delta_m  = (2 + M) \frac{2 \lambda \sqrt{k}}{\alpha_1 m}$.  Here $\tau_1(\alpha_1, \kappa, \sigma, \Sigma )$ is a constant which is independent of $p, k, m$ and $n$.
}
\begin{proof}
	The optimization problem~\eqref{eq:underline theta and Jstar} can be rewritten in the following way:
\begin{align*}
	\begin{split}
		\underline{\theta} =& \arg\min_{\widetilde{\theta} \in \real^k} \sum_{(X_i, y_i) \in J^*} (y_i - \inner{\widetilde{X}_i}{\widetilde{\theta}})^2 +  \lambda (\|\widetilde{\theta}\|_1 + 1)^2 \;.
	\end{split}  
\end{align*}
Then $\underline{\theta}$ must satisfy the stationarity KKT condition, i.e., 
\begin{align}
	\label{eq:underline theta for kkt}
	- \sum_{(X_i, y_i) \in J^*}\widetilde{X}_i (y_i - \inner{\widetilde{X}_i}{\underline{\theta}}) + \lambda \omega  (\inner{\underline{\theta}}{\omega} + 1) = 0 \;,
\end{align} 
where $\omega$ is an element of the subdifferential set of $\| \widetilde{\theta} \|_1$ at $\underline{\theta}$. Since $(X-i, y_i) \in J^*$, we can substitute $y_i = \inner{\widetilde{X}_i}{\underline{\theta}^*} + e_i$. With little manipulation and by use of norm-inequalities, we can rewrite this as
\begin{align}
	\| \underline{\theta} - \underline{\theta}^*  \|_2 \leq \| \left(\widehat{\Sigma}_{SS}^{J^*} + \frac{\lambda}{m} \omega \omega^T\right)^{-1}  \|_2 \left( \| \frac{1}{m} \sum_{(X_i, y_i) \in J^*} \widetilde{X}_i e_i \|_2 + \| \frac{1}{m} \lambda \omega \omega^T \underline{\theta}^* \|_2 + \| \frac{\lambda}{m} \omega \|_2 \right)
\end{align}
We observe that by using $ m \geq \Omega(\frac{k^3 \log^2 p}{\tau_1(\alpha_1, \kappa, \sigma, \Sigma )})$ and Weyl's inequality, we have $\| \left(\widehat{\Sigma}_{SS}^{J^*} + \frac{\lambda}{m} \omega \omega^T\right)^{-1}  \|_2 \leq \frac{2}{\alpha_1}$. Thus,
\begin{align}
	\| \underline{\theta} - \underline{\theta}^*  \|_2 \leq \frac{2}{\alpha_1} \left( \| \frac{1}{m} \sum_{(X_i, y_i) \in J^*} \widetilde{X}_i e_i \|_2 + \frac{\lambda \sqrt{k}}{m}  \|  \underline{\theta}^* \|_1 +  \frac{\lambda \sqrt{k}}{m}  \right)
\end{align}
Next we bound $\| \frac{1}{m} \sum_{(X_i, y_i) \in J^*} \widetilde{X}_i e_i \|_2$.

\begin{lemma}
	\label{lem:Xe l2}
	If $\lambda \geq 8 \sigma \sigma_e \sqrt{m \log p}$, then $\| \frac{1}{m} \sum_{(X_i, y_i) \in J^*} \widetilde{X}_i e_i \|_2 \leq \sqrt{k} \frac{\lambda}{m}$ with probability at least $1 - \calO(\frac{1}{p})$.
\end{lemma}
\begin{proof}
	We take the $i$-th entry of $\frac{1}{m} \sum_{(X_j, y_j) \in J^*} \widetilde{X}_j e_j$ for some $i \in S$, i.e., $| \frac{1}{n} \sum_{(X_j, y_j) \in J^*} X_{ji} e_j |$. 		
	Recall that $X_{ji}$ is a sub-Gaussian random variable with parameter $\sigma$ and $e_j$ is a sub-Gaussian random variable with parameter $\sigma_e$. Then, $\frac{X_{ji}}{\sigma}\frac{e_j}{\sigma_e}$ is a sub-exponential random variable with parameters $(4\sqrt{2}, 2)$. Using the concentration bounds for the sum of independent sub-exponential random variables~\citep{wainwright2019high}, we can write:
	\begin{align}
		\begin{split}
			\prob( | \frac{1}{m} \sum_{(X_j, y_j) \in J^*} \frac{X_{ji}}{\sigma}\frac{e_j}{\sigma_e} | \geq t) \leq 2 \exp(- \frac{m t^2}{64}), \; 0 \leq t \leq 8
		\end{split}
	\end{align} 
	Taking a union bound across $i \in S$:
	\begin{align}
		\begin{split}
			&\prob( \exists i \in S \mid | \frac{1}{m} \sum_{(X_j , y_j) \in J^*} \frac{X_{ji}}{\sigma}\frac{e_j}{\sigma_e} | \geq t) \leq 2 k \exp(- \frac{m t^2}{64}),\; 0 \leq t \leq 8
		\end{split}
	\end{align}
	
	It follows that $\| \frac{1}{m} \sum_{(X_i, y_i) \in J^*} \widetilde{X}_i e_i \|_2 \leq \sqrt{k} t$ with probability at least $1 - 2k \exp(- \frac{mt^2}{64 \sigma^2\sigma_e^2})$ for some $0 \leq t \leq 8 \sigma \sigma_e$. Taking $t = \frac{\lambda}{m}$, we get the desired result.
\end{proof}
Using the above bound it follows that $\| \underline{\theta} - \underline{\theta}^* \|_2 \leq (2 + M) \frac{2 \lambda \sqrt{k}}{\alpha_1 m}$.
\end{proof}

\subsection{Proof of Lemma~\ref{lem:nu is feasible}}
\paragraph{Lemma~\ref{lem:nu is feasible}}
\emph{Under assumptions~\ref{assum:positive definite Hessian} and \ref{assum:mutual incoherence} and if we choose $\lambda \geq 8 \sigma \sigma_e \sqrt{m \log p}$ and  $ m \geq \Omega(\frac{k^3 \log^2 p}{\tau_1(\alpha_1, \kappa, \sigma, \Sigma )})$, setting of $\nu$ in equation~\eqref{eq:feasibility of nu} is feasible with probability at least $1 - \calO(\frac{1}{p})$ for sufficiently large $\rho$.  Here $\tau_1(\alpha_1, \kappa, \sigma, \Sigma )$ is a constant which is independent of $p, k, m$ and $n$.
}
\begin{proof}
	It suffices to show that for $\underline{\theta}$ coming from optimization problem~\eqref{eq:underline theta and Jstar}, for all $(X_c, y_c) \in \calM_c$ and for all $(X_o, y_o) \in \calM_o$ we have 
	\begin{align}
		f(\tilde{X}_c, y_c, \underline{\theta}) \leq f(\tilde{X}_o, y_o, \underline{\theta}).
	\end{align}
	We start with the following.
	\begin{align}
		\begin{split}
			f(\tilde{X}_o, y_o, \underline{\theta}) - f(\tilde{X}_c, y_c, \underline{\theta}) =& (y_o - \inner{\tilde{X}_o}{\underline{\theta}})^2 - (y_c - \inner{\tilde{X}_c}{\underline{\theta}})^2 \\
			&= (y_o - \inner{\tilde{X}_o}{ \underline{\theta}^* - \underline{\theta}^* + \underline{\theta} })^2 - (y_c - \inner{\tilde{X}_c}{ \underline{\theta}^* - \underline{\theta}^* + \underline{\theta}})^2\\
			&\geq (y_o - \inner{\tilde{X}_o}{\underline{\theta}^*})^2 - (y_c - \inner{\tilde{X}_c}{\underline{\theta}^*})^2  - ( - \underline{\theta}^* + \underline{\theta})^T \widetilde{X}_c \widetilde{X}_c^T \\
            &( - \underline{\theta}^* + \underline{\theta}) + 2 (y_o - \inner{\widetilde{X}_o}{\theta^*}) \widetilde{X}_o^T(\underline{\theta} - \underline{\theta}^* ) - 2 (y_c - \inner{\widetilde{X}_c}{\theta^*}) \widetilde{X}_c^T \\
            &(\underline{\theta} - \underline{\theta}^* ) \\
			&\geq \rho - m \alpha_2 \| \underline{\theta} - \underline{\theta}^* \|^2 - 2  \| (y_o - \inner{\widetilde{X}_o}{\theta^*}) \widetilde{X}_o \| \| \underline{\theta} - \underline{\theta}^* \| \\
            &- 2 e_c \widetilde{X}_c^T ( - \underline{\theta}^* + \underline{\theta})
		\end{split}
	\end{align}  

Next, we bound $| e_c \widetilde{X}_c^T ( - \underline{\theta}^* + \underline{\theta})|$.
\begin{lemma}
	For fixed $\|( - \underline{\theta}^* + \underline{\theta})  \|_2$, $\prob(| e_c \widetilde{X}_c^T ( - \underline{\theta}^* + \underline{\theta})| \leq \frac{\rho}{4})$ with probability at least $1 - \calO(\frac{1}{p})$. 
\end{lemma}
\begin{proof}
	 Let $\Delta$ be $- \underline{\theta}^* + \underline{\theta}) $, then $\widetilde{X}_i^\T \Delta$ is a sub-Gaussian random variable with parameter $\sigma \| \Delta_1 \|$ and $e_i$ is a sub-Gaussian random variable with parameter $\sigma_e$.
	Then, $\frac{\widetilde{X}_i^T \Delta}{\sigma \| \Delta \|_2} \frac{e_i}{\sigma_e}$ is a sub-exponential random variable with parameters $(4\sqrt{2}, 2)$. Using the concentration bounds for the sum of independent sub-exponential random variables~\citep{wainwright2019high}, we can write:
	\begin{align}
		\begin{split}
			\prob( | \frac{\widetilde{X}_i^T \Delta}{\sigma \| \Delta \|_2}\frac{e_i}{\sigma_e} | \geq t) \leq 2 \exp(- \frac{t^2}{64}), \; 0 \leq t \leq 8
		\end{split}
	\end{align}  
	Taking $t = \frac{t}{\sigma \| \Delta \|_2\sigma_e }$, we get
	\begin{align}
		\begin{split}
			\prob( | e_i X_{i_P}^\T \Delta | \geq t) \leq 2 \exp(- \frac{t^2}{64 \sigma^2 \| \Delta \|_2^2 \sigma_e^2}), \; 0 \leq t \leq 8\sigma \| \Delta \|_2\sigma_e
		\end{split}
	\end{align}
	We take $t = \frac{\rho}{4}$, then
	\begin{align}
		\begin{split}
			\prob( | e_i \widetilde{X}_i^\T \Delta | \geq \frac{\rho}{4}) \leq 2 \exp(- \frac{\rho^2}{16 \times 64 \sigma^2 \| \Delta \|_2^2 \sigma_e^2}), \; 0 \leq \rho \leq 32\sigma \| \Delta \|_2\sigma_e
		\end{split}
	\end{align}
	Since  $\| \Delta \|_2$ is upper bounded with $\calO( \frac{\lambda}{m} \sqrt{k} )$ and $m$ is of order $\calO(k^3 \log^2 p)$, thus $\prob( | e_i \widetilde{X}_i^\T \Delta | \leq \frac{\rho}{4})$ with probability at least $1 - \calO(\frac{1}{p})$. 
\end{proof}

Thus as long as 
\begin{align}
	\rho \geq \frac{4}{3}\left( m \alpha_2 \delta_m^2 + 2  \| (y_o - \inner{\widetilde{X}_o}{\theta^*}) \widetilde{X}_o \| \delta_m  \right) \;,
\end{align}
the setting of $\nu$ is feasible with probability at least $1 - \calO(\frac{1}{p})$ .

\end{proof}

\subsection{Proof of Lemma~\ref{lem:zero eigenvalue}}

\paragraph{Lemma~\ref{lem:zero eigenvalue}}
	\emph{The dual variable $\Lambda$ from Theorem~\ref{thm:kkt conditions} has zero eigenvalue corresponding to $\begin{bmatrix}
		\underline{\theta} \\ 1
	\end{bmatrix}$.
}
\begin{proof}
	\label{proof:zero eigenvalue}
	The optimization problem~\eqref{eq:underline theta and Jstar} can be rewritten in the following way:
	$ \underline{\theta} = \arg\min_{\widetilde{\theta} \in \real^k} \sum_{(X_i, y_i) \in J^*} (y_i - \inner{\widetilde{X}_i}{\widetilde{\theta}})^2 +  \lambda (\|\widetilde{\theta}\|_1 + 1)^2$. 
	Then $\underline{\theta}$ must satisfy the stationarity KKT condition, i.e., 
	\begin{align}
		\label{eq:underline theta for kkt}
		- \sum_{(X_i, y_i) \in J^*}\widetilde{X}_i (y_i - \inner{\widetilde{X}_i}{\underline{\theta}}) + \lambda \omega  (\inner{\underline{\theta}}{\omega} + 1) = 0 \;,
	\end{align} 
	where $\omega$ is an element of the subdifferential set of $\| \widetilde{\theta} \|_1$. By construction, we treat $\zeta = \begin{bmatrix}
		\omega \\ 1
	\end{bmatrix} \begin{bmatrix}
	\omega^T & 1
\end{bmatrix}$. After little algebraic manipulation, we have $\left(\sum_{(X_i, y_i) \in J^*} \widetilde{A}_i + \lambda \zeta + \mu\right) \begin{bmatrix}
		\underline{\theta} \\ 1
	\end{bmatrix} = 0 \implies \Lambda \begin{bmatrix}
	\underline{\theta} \\ 1
\end{bmatrix} = 0 \;$. This completes our proof.
\end{proof} 

\subsection{Proof of Lemma~\ref{lem:second eigenvalue}}
\paragraph{Lemma~\ref{lem:second eigenvalue}}
\emph{Under assumption \ref{assum:mutual incoherence} and by choosing $m \geq \Omega(\frac{k + \log p}{\alpha_1^2})$, we ensure that second minimum eigenvalue of $\Lambda$ is strictly positive with probability at least $1 - \calO(\frac{1}{p})$. 
}
\begin{proof}
	We know that
	\begin{align}
		\begin{split}
			\Lambda &= \sum_{(X_i, y_i) \in J^*} \widetilde{A}_i + \lambda \zeta + \mu \\
			&= \sum_{(X_i, y_i) \in J^*} \begin{bmatrix} \widetilde{X}_i \widetilde{X}_i^T & - \widetilde{X}_i y_i \\ - y_i \widetilde{X}_i^T & y_i^2 \end{bmatrix} + \lambda \begin{bmatrix} \omega \omega^T & \omega \\ \omega^T & 1 \end{bmatrix} + \mu \\
			&= \begin{bmatrix} \sum_{(X_i, y_i) \in J^*} \widetilde{X}_i \widetilde{X}_i^T + \lambda \omega \omega^T & \sum_{(X_i, y_i) \in J^*} - \widetilde{X}_i y_i + \lambda \omega \\
				\sum_{(X_i, y_i) \in J^*} - y_i \widetilde{X}_i^T + \lambda_1 \omega^T & \sum_{(X_i, y_i) \in J^*} y_i^2 + \lambda + \mu_{k+1, k+1}
			\end{bmatrix}
		\end{split}
	\end{align}
	Also note that $\mu_{k+1, k+1} =  - \inner{ \sum_{(X_i, y_i) \in J^*} \widetilde{A}_i + \lambda \zeta }{\uvartheta} = - \sum_{(X_i, y_i) \in J^*}  (y_i - \widetilde{X}_i^T \underline{\theta} )^2 + \lambda (\| \underline{\theta} \|_1 + 1)^2   $. We also know that $\underline{\theta}$ satisfies the stationarity KKT condition, i.e.,
	\begin{align*}
		\begin{split}
			&\sum_{(X_i, y_i) \in J^*} \widetilde{X}_i (- y_i + \widetilde{X}_i^\T \underline{\theta}) + \omega \lambda (\omega^T \underline{\theta} + 1) = 0 \\
			&\underline{\theta}  = - ( \sum_{(X_i, y_i) \in J^*} \widetilde{X}_i \widetilde{X}_i^T + \lambda \omega \omega^T)^{-1} (\sum_{(X_i, y_i) \in J^*} - \widetilde{X}_i y_i + \lambda \omega)
		\end{split}
	\end{align*} 
	Using the stationarity KKT condition, we can simplify objective function value to $  \sum_{(X_i, y_i) \in J^*} y_i^2 +\\ (\sum_{(X_i, y_i) \in J^*} - y_i \widetilde{X}_i^T + \lambda \omega^T) \underline{\theta} + \lambda$. 	Now, we invoke Haynesworth's inertia additivity formula~\citep{haynsworth1968determination} to prove our claim. Let $R$ be a block matrix of the form $ R = \begin{bmatrix} A & B \\ B^T & C \end{bmatrix}$, then inertia of matrix $R$, denoted by $\inertia(R)$, is defined as the tuple $(|\eig_{+}(R)|, |\eig_{-}(R)|, |\eig_0(R)| )$ where $|\eig_+(R)|$ is the number of positive eigenvalues, $|\eig_-(R)|$ is the number of negative eigenvalues and $|\eig_0(R)|$ is the number of zero eigenvalues of matrix $R$. Haynesworth's inertia additivity formula is given as:
	\begin{align}
		\label{eq:Haynesworth inertia additivity formula}
		\begin{split}
			\inertia(R) = \inertia(A) + \inertia(C - B^T A^{-1} B)
		\end{split}
	\end{align}
	We take $A =  \sum_{(X_i, y_i) \in J^*} \widetilde{X}_i \widetilde{X}_i^T + \lambda \omega \omega^T  $, $B = \sum_{(X_i, y_i) \in J^*} - \widetilde{X}_i y_i + \lambda \omega $ and $C = \sum_{(X_i, y_i) \in J^*} y_i^2 + \lambda + \mu_{k+1, k+1}$. It should be noted that $C - B^T A^{-1} B$ evaluates to zero. Thus,
	\begin{align}
		\begin{split}
			\inertia(\Lambda) = \inertia(\sum_{(X_i, y_i) \in J^*} \widetilde{X}_i \widetilde{X}_i^T + \lambda \omega \omega^T ) + \inertia(0)
		\end{split}
	\end{align} 
	We note that $0$ has precisely one zero eigenvalue and no other eigenvalues. Moreover, with large enough $m$ and Weyl's inequality:
	\begin{align}
		\begin{split}
			\eig_{\min}(\sum_{(X_i, y_i) \in J^*} \widetilde{X}_i \widetilde{X}_i^T + \lambda \omega \omega^T) \geq \frac{\alpha_1}{2} > 0
		\end{split}
	\end{align} 
	with probability at least $1 - \calO(\frac{1}{p})$ as long as $m = \Omega(\frac{k + \log p}{\alpha_1^2})$. It follows that the second eigenvalue of $\Lambda$ is strictly positive.
\end{proof}

\subsection{Proof of Lemma~\ref{lem:bound omegabar}}

\paragraph{Lemma~\ref{lem:bound omegabar}}
\emph{Under assumptions~\ref{assum:positive definite Hessian}, \ref{assum:mutual incoherence} and after choosing $\lambda \geq \Omega(\sqrt{m \log p})$ and $m \geq \Omega(k^3 \log^2 p)$, the following bound on $\bar{\omega}$ from \eqref{eq:expression for omega bar} holds with high probability:
	\begin{align}
		\| \bar{\omega} \|_{\infty} \leq 1 - \frac{\kappa}{4}
	\end{align} 
}
\begin{proof}
	We start by rewriting equation~\eqref{eq:expression for omega bar}.
	\begin{align}
		\begin{split}
			&\frac{\lambda}{m}(1 + \| \widehat{\theta} \|_1) \bar{\omega} = - \widehat{\Sigma}^{J^*}_{S^cS} \widehat{\Sigma}^{J^* -1}_{SS} \left( \frac{1}{m} \sum_{(X_i, y_i) \in J^*} \widetilde{X}_i e_i  - \frac{\lambda}{m}(1 + \| \widehat{\theta} \|_1) \widetilde{\omega} \right) +  \frac{1}{m} \sum_{(X_i, y_i) \in J^*} \overline{X}_i e_i \;,
		\end{split} 
	\end{align}
which implies
	\begin{align}
	\begin{split}
		&\| \bar{\omega} \|_\infty \leq \| \widehat{\Sigma}^{J^*}_{S^cS} \widehat{\Sigma}^{J^* -1}_{SS} \|_{\infty, \infty} \left( \| \frac{1}{\lambda m} \sum_{(X_i, y_i) \in J^*} \widetilde{X}_i e_i \|_{\infty} +   \| \widetilde{\omega} \|_\infty \right) +  \| \frac{1}{\lambda m} \sum_{(X_i, y_i) \in J^*} \overline{X}_i e_i \|_\infty \;,
	\end{split} 
\end{align}
We provide following lemma to bound terms in right hand side.

\begin{lemma}
	\label{lem:bound X_se and X_s^ce}
	Let $\lambda \geq \frac{64 \sigma \sigma_e}{\kappa} \sqrt{m \log p}$. Then the following holds true:
	\begin{align*}
		\begin{split}
			&\prob(\| \frac{1}{\lambda } \frac{1}{m} \sum_{(X_i, y_i) \in J^*} \widetilde{X}_i e_i \|_{\infty} \geq \frac{\kappa}{8 - 4\kappa}) \leq \calO(\frac{1}{p}),\\
			& \prob(\| \frac{1}{\lambda} \frac{1}{m} \sum_{(X_i, y_i) \in J^*} \overline{X}_i e_i \|_{\infty} \geq \frac{\kappa}{8}) \leq \calO(\frac{1}{p}) 
		\end{split}
	\end{align*}
\end{lemma}
\begin{proof}
	\label{proof:bound X_se and X_s^ce}
	We will start with $ \frac{1}{m} \sum_{(X_i, y_i) \in J^*} \widetilde{X}_i e_i$. We take the $i$-th entry of $\frac{1}{m} \sum_{(X_j, y_j) \in J^*} \widetilde{X}_j e_j$ for some $i \in S$, i.e., $| \frac{1}{m} \sum_{(X_j, y_j) \in J^*} \widetilde{X}_{ji} e_j |$. 		
	Recall that $\widetilde{X}_{ji}$ is a sub-Gaussian random variable with parameter $\sigma$ and $e_j$ is a sub-Gaussian random variable with parameter $\sigma_e$. Then, $\frac{\widetilde{X}_{ji}}{\sigma}\frac{e_j}{\sigma_e}$ is a sub-exponential random variable with parameters $(4\sqrt{2}, 2)$. Using the concentration bounds for the sum of independent sub-exponential random variables~\citep{wainwright2019high}, we can write:
	\begin{align}
		\begin{split}
			\prob( | \frac{1}{m} \sum_{(X_j, y_j) \in J^*} \frac{\widetilde{X}_{ji}}{\sigma}\frac{e_j}{\sigma_e} | \geq t) \leq 2 \exp(- \frac{m t^2}{64}), \; 0 \leq t \leq 8
		\end{split}
	\end{align} 
	Taking a union bound across $i \in S$:
	\begin{align}
		\begin{split}
			&\prob( \exists i \in S \mid | \frac{1}{m} \sum_{(X_j, y_j) \in J^*} \frac{\widetilde{X}_{ji}}{\sigma}\frac{e_j}{\sigma_e} | \geq t) \leq 2 k \exp(- \frac{m t^2}{64}), \; 0 \leq t \leq 8
		\end{split}
	\end{align}
	
	Taking $t = \frac{\lambda t}{ \sigma \sigma_e}$, we get: 
	\begin{align}
		\begin{split}
			&\prob( \exists i \in S \mid |\frac{1}{\lambda} \frac{1}{m} \sum_{(X_j, y_j) \in J^*} X_{ji} e_j | \geq t) \leq 2 k \exp(- \frac{m \lambda^2 t^2}{64 \sigma^2 \sigma_e^2}), \; 0 \leq t \leq 8 \frac{\sigma \sigma_e}{\lambda}
		\end{split}
	\end{align}
	
	It follows that $\| \frac{1}{\lambda } \frac{1}{m} \sum_{(X_i, y_i) \in J^*} \widetilde{X}_i e_i \|_{\infty} \leq  t$ with probability at least $1 - 2 k \exp(- \frac{m \lambda^2 t^2}{64 \sigma^2 \sigma_e^2})$.
	
	Using a similar argument, we can show that $\| \frac{1}{\lambda} \frac{1}{m} \sum_{(X_i, y_i) \in J^*} \overline{X}_{i} e_i \|_{\infty} \leq  t$ with probability at least $1 - 2(p - k) \exp(- \frac{m \lambda^2 t^2}{64 \sigma^2 \sigma_e^2})$. Taking $t= \frac{\kappa}{8 - 4 \kappa}$ and $\frac{\kappa}{8}$ in the first and second inequality of Lemma \ref{lem:bound X_se and X_s^ce} and choosing the provided setting of $\lambda$ and $m$ completes our proof.
\end{proof}

\end{proof}

\section{Auxiliary Lemmas and Proofs}
\label{sec:auxiliary lemmas and proofs}

\begin{lemma}
\label{lem:sample positive definite}
If Assumption \ref{assum:positive definite Hessian} holds and $m = \Omega(\frac{k + \log p}{\alpha_1^2})$, then $\frac{\alpha_1}{2} I \preceq \widehat{\Sigma}_{SS}^J \preceq  \frac{3\alpha_2}{2} I$ with probability at least $1 - \calO(\frac{1}{p})$.
\end{lemma}
\begin{proof}
	By the Courant-Fischer variational representation~\citep{horn2012matrix}:
	\begin{align}
		\begin{split}
			\eig_{\min}(\Sigma_{SS}) = \min_{\|y \|_2 = 1} y^T \Sigma_{SS} y &= \min_{\|y \|_2 = 1} y^T (\Sigma_{SS} - \widehat{\Sigma}^J_{SS}+ \widehat{\Sigma}^J_{SS}) y \\
			&\leq y^T (\Sigma_{SS})  - \widehat{\Sigma}^J_{SS}+ \widehat{\Sigma}^J_{SS} ) y \\
			&= y^\T (\Sigma_{SS} - \widehat{\Sigma}^J_{SS}) y + y^\T \widehat{\Sigma}^J_{SS}  y
		\end{split}
	\end{align}
	It follows that
	\begin{align}
		\begin{split}
			\eig_{\min}(\widehat{\Sigma}^J_{SS}) \geq \alpha_1 - \| \Sigma_{SS} - \widehat{\Sigma}^J_{SS} \|_2
		\end{split}
	\end{align}
	The term $ \| \Sigma_{SS} - \widehat{\Sigma}^J_{SS} \|_2$ can be bounded using Proposition 2.1 in \cite{vershynin2012close} for sub-Gaussian random variables. In particular,
	\begin{align}
		\begin{split}
			\prob(\|  \Sigma_{SS} - \widehat{\Sigma}^J_{SS} \|_2 \geq \epsilon) \leq 2  \exp(-c\epsilon^2 m + k) 
		\end{split}
	\end{align}
	for some constant $c > 0$. Taking $\epsilon = \frac{\alpha_1}{2}$, we show that $\eig_{\min}(\widehat{\Sigma}^J_{SS}) \geq \frac{\alpha_1}{2}$ with probability at least $1 - 2  \exp(- \frac{c\alpha_1^2 m}{4} + k)$. 
	
	Remark: Similarly, it can be shown that $\eig_{\max}(\widehat{\Sigma}^J_{SS}) \leq \frac{3\alpha_2}{2}$ with probability at least $1 - 2  \exp(- \frac{c\alpha_2^2 m}{4} + k)$.
\end{proof}

\begin{lemma}
\label{lem:sample mutual incoherence condition}
If Assumption \ref{assum:mutual incoherence} holds and $m = \Omega(\frac{k^3 (\log k + \log p)}{\tau(\alpha_1, \kappa, \sigma, \Sigma)})$, then $ \| \widehat{\Sigma}_{S^cS}^J \widehat{\Sigma}_{SS}^{J -1 } \|  \leq 1 - \frac{\kappa}{2}$ with probability at least $1 - \calO(\frac{1}{p})$ where $\tau(\alpha_1, \kappa, \sigma, \Sigma)$ is a constant independent of $m, n, p$ and $k$.
\end{lemma}
\begin{proof}
	Before we prove the result of Lemma \ref{lem:sample mutual incoherence condition}, we will prove a helper lemma. 
	\begin{lemma}
		\label{lem:helper mutual incoherence}
		If Assumption \ref{assum:mutual incoherence} holds then for some $\delta > 0$, the following inequalities hold:
		\begin{align}
			\label{eq:helper mutual incoherence}
			\begin{split}
				&\prob( \| \widehat{\Sigma}^J_{S^cS} - \Sigma_{S^cS} \|_{\infty} \geq \delta ) \leq 4 (p - k) k \exp( - \frac{m \delta^2}{128 k^2 (1+4\sigma^2) \max_{l} \Sigma_{ll}^2}) \\
				&\prob( \| \widehat{\Sigma}^J_{SS} - \Sigma_{SS} \|_{\infty} \geq \delta ) \leq 4 k^2 \exp( - \frac{m \delta^2}{128 k^2 (1+4\sigma^2) \max_{l} \Sigma_{ll}^2}) \\
				&\prob( \| (\widehat{\Sigma}^J_{SS})^{-1} - (\Sigma_{SS})^{-1} \|_{\infty} \geq \delta ) \leq 2\exp(- \frac{c\delta^2 \alpha_1^4 m}{4 k} + k)+ 2 \exp( - \frac{c\alpha_1^2 m}{4} + k)
			\end{split}
		\end{align}
	\end{lemma}
	\begin{proof}
		\label{proof:helper mutual incoherence}
		Let $A_{ij}$ be $(i, j)$-th entry of $\widehat{\Sigma}^J_{S^cS} - \Sigma_{S^cS}$. Clearly, $\E(A_{ij}) = 0$. By using the definition of the $\| \cdot\|_{\infty}$ norm, we can write:
		\begin{align}
			\begin{split}
				\prob(\| \widehat{\Sigma}^J_{S^cS} - \Sigma_{S^cS} \|_{\infty} \geq \delta) &= \prob(\max_{i \in S^c} \sum_{j \in S} |A_{ij}| \geq \delta) \\
				& \leq (p - k) \prob(\sum_{j \in S} |A_{ij}| \geq \delta) \\
				&\leq (p - k) k \prob(|A_{ij}| \geq \frac{\delta}{k})
			\end{split}
		\end{align}
		where the second last inequality comes as a result of the union bound across entries in $S^c$ and the last inequality is due to the union bound across entries in $S$. Recall that $X_i, i \in \{1,\cdots,p\}$ are zero mean random variables with covariance $\Sigma$ and each $\frac{X_i}{\sqrt{\Sigma_{ii}}}$ is a sub-Gaussian random variable with parameter $\sigma$. Using the results from Lemma 1 of \cite{ravikumar2010high}, for some $\delta \in (0, k \max_{l} \Sigma_{ll} 8(1 + 4 \sigma^2))$, we can write:
		\begin{align}
			\begin{split}
				\prob(|A_{ij}| \geq \frac{\delta}{k}) \leq 4 \exp( - \frac{m \delta^2}{128 k^2 (1+4\sigma^2) \max_{l} \Sigma_{ll}^2})
			\end{split}
		\end{align} 
		Therefore,
		\begin{align}
			\begin{split}
				&\prob(\| \widehat{\Sigma}^J_{S^cS} - \Sigma_{S^cS} \|_{\infty} \geq \delta)  \leq 4 (p - k) k \exp( - \frac{m \delta^2}{128 k^2 (1+4\sigma^2) \max_{l} \Sigma_{ll}^2})
			\end{split}
		\end{align}
		Similarly, we can show that 
		\begin{align}
			\begin{split}
				&\prob(\| \widehat{\Sigma}^J_{SS} - \Sigma_{SS} \|_{\infty} \geq \delta)  \leq 4 k^2 \exp( - \frac{m \delta^2}{128 k^2 (1+4\sigma^2) \max_{l} \Sigma_{ll}^2})
			\end{split}
		\end{align}
		Next, we will show that the third inequality in \eqref{eq:helper mutual incoherence} holds. Note that
		\begin{align}
			\begin{split}
				\| (\widehat{\Sigma}^J_{S^cS})^{-1} - (\Sigma_{S^cS})^{-1} \|_{\infty} &= \| (\Sigma_{SS})^{-1} (\Sigma_{SS} - \widehat{\Sigma}^J_{SS}) (\widehat{\Sigma}^J_{SS})^{-1} \|_{\infty} \\
				&\leq \sqrt{k}  \| (\Sigma_{SS})^{-1} (\Sigma_{SS} - \widehat{\Sigma}^J_{SS}) (\widehat{\Sigma}^J_{SS})^{-1} \|_{2} \\
				&\leq \sqrt{k}  \| (\Sigma_{SS})^{-1} \|_2 \| (\Sigma_{SS} - \widehat{\Sigma}^J_{SS})\|_2 \| (\widehat{\Sigma}^J_{SS})^{-1} \|_{2} \\
			\end{split}
		\end{align}   
		Note that $\| \Sigma_{SS} \|_2 \geq \alpha_1$, thus $\| (\Sigma_{SS})^{-1} \|_2 \leq \frac{1}{\alpha_1}$. Similarly,  $\| \Sigma_{SS} \|_2 \geq \frac{\alpha_1}{2}$ with probability at least $1 - 2 \exp( - \frac{c\alpha_1^2m}{4} + k)$. We also have $\| (\Sigma_{SS} - \widehat{\Sigma}^J_{SS})\|_2 \leq \epsilon$ with probability at least $1 - 2\exp(-c\epsilon^2 m + k)$. Taking $\epsilon = \delta \frac{\alpha_1^2}{2 \sqrt{k}}$, we get 
		\begin{align}
			\begin{split}
				\prob( \| (\Sigma_{SS} - \widehat{\Sigma}^J_{SS})\|_2 \geq  \delta \frac{\alpha_1^2}{2 \sqrt{k}} ) \leq 2\exp(- \frac{c\delta^2 \alpha_1^4 m}{4 k} + k)
			\end{split}
		\end{align}
		It follows that $\| (\widehat{\Sigma}^J_{SS})^{-1} - (\Sigma_{SS})^{-1} \|_{\infty} \leq \delta$ with probability at least $1 - 2\exp(- \frac{c\delta^2 \alpha_1^4 m}{4 k} + k) - 2 \exp( - \frac{cC_{\min}^2 m}{4} + k)$.
	\end{proof}
	Now we are ready to show that the statement of Lemma \ref{lem:sample mutual incoherence condition} holds using the results from Lemma \ref{lem:helper mutual incoherence}. We will rewrite $\widehat{\Sigma}^J_{S^cS} (\widehat{\Sigma}^J_{SS})^{-1}$ as the sum of four different terms:
	\begin{align}
		\begin{split}
			\widehat{\Sigma}^J_{S^cS} (\widehat{\Sigma}^J_{SS})^{-1} = T_1 + T_2 + T_3 + T_4,
		\end{split}
	\end{align} 
	where 
	\begin{align}
		\begin{split}
			T_1 &\triangleq \widehat{\Sigma}^J_{S^cS} ( (\widehat{\Sigma}^J_{SS})^{-1} - (\Sigma_{SS})^{-1} ) \\
			T_2 &\triangleq (\widehat{\Sigma}^J_{S^cS} - \Sigma_{S^cS}) (\Sigma_{SS})^{-1} \\
			T_3 &\triangleq (\widehat{\Sigma}^J_{S^cS} - \Sigma_{S^cS})((\widehat{\Sigma}^J_{SS})^{-1} - (\Sigma_{SS})^{-1}) \\
			T_4 &\triangleq \Sigma_{S^cS} (\Sigma_{SS})^{-1}   \, .
		\end{split}
	\end{align}
	Then it follows that  $\| \widehat{\Sigma}^J_{S^cS} (\widehat{\Sigma}^J_{SS})^{-1} \|_{\infty} \leq \| T_1 \|_{\infty} + \| T_2 \|_{\infty} + \| T_3 \|_{\infty} + \| T_4 \|_{\infty}$. Now, we will bound each term separately. First, recall that Assumption \ref{assum:mutual incoherence} ensures that $\| T_4 \|_{\infty} \leq 1 - \kappa$.
	\paragraph{Controlling $T_1$.} We can rewrite $T_1$ as,
	\begin{align}
		\begin{split}
			T_1 = - \Sigma_{S^cS} (\Sigma_{SS})^{-1} (\widehat{\Sigma}^J_{SS} - \Sigma_{SS}) (\widehat{\Sigma}^J_{SS})^{-1}
		\end{split}
	\end{align}
	then,
	\begin{align}
		\begin{split}
			\| T_1 \|_{\infty} &= \| \Sigma_{S^cS} (\Sigma_{SS})^{-1} (\widehat{\Sigma}^J_{SS} - \Sigma_{SS}) (\widehat{\Sigma}^J_{SS})^{-1} \|_{\infty} \\
			&\leq \| \Sigma_{S^cS} (\Sigma_{SS})^{-1} \|_{\infty} \| (\widehat{\Sigma}^J_{SS} - \Sigma_{SS})\|_{\infty} \| (\widehat{\Sigma}^J_{SS})^{-1} \|_{\infty} \\
			&\leq (1 - \kappa)  \| (\widehat{\Sigma}^J_{SS} - \Sigma_{SS})\|_{\infty} \sqrt{k} \| (\widehat{\Sigma}^J_{SS})^{-1} \|_2 \\
			&\leq (1 - \kappa) \| (\widehat{\Sigma}^J_{SS} - \Sigma_{SS})\|_{\infty} \frac{2\sqrt{k}}{\alpha_1} \\
			&\leq \frac{\kappa}{6} 
		\end{split}
	\end{align}
	The last inequality holds with probability at least $1 - 2\exp( - \frac{c \alpha_1^2 m}{4} + k) - 4 k^2 \exp( - \frac{m \alpha_1^2 \kappa^2 }{18432(1-\kappa)^2 k^3(1 + 4\sigma^2) \max_{l} \Sigma_{ll}^2} )$ by taking $\delta = \frac{\alpha_1 \kappa}{12 (1 - \kappa) \sqrt{k}}$.
	\paragraph{Controlling $T_2$.} Recall that $T_2 = (\widehat{\Sigma}^J_{S^cS} - \Sigma_{S^cS}) (\Sigma_{SS})^{-1}$. Thus,
	\begin{align}
		\begin{split}
			\| T_2 \|_{\infty} &\leq \sqrt{k} \| (\Sigma_{SS})^{-1} \|_2 \| (\widehat{\Sigma}^J_{S^cS} - \Sigma_{S^cS}) \|_{\infty} \\
			&\leq \frac{\sqrt{k}}{\alpha_1}   \| (\widehat{\Sigma}^J_{S^cS} - \Sigma_{S^cS}) \|_{\infty} \\
			&\leq \frac{\kappa}{6}
		\end{split}
	\end{align}
	The last inequality holds with probability at least $1 - 4(p-k)k \exp( -\frac{m \alpha_1^2 \kappa^2 }{4608 k^3 (1 + 4\sigma^2) \max_{l} \Sigma_{ll}^2 } )$ by choosing $\delta = \frac{\alpha_1 \kappa}{6 \sqrt{k}}$.
	\paragraph{Controlling $T_3$.} Note that,
	\begin{align}
		\begin{split}
			\| T_3 \|_{\infty} &\leq \| (\widehat{\Sigma}^J_{S^cS} - \Sigma_{S^cS}) \|_{\infty} \| ((\widehat{\Sigma}^J_{SS})^{-1} - (\Sigma_{SS})^{-1}) \|_{\infty} \\
			&\leq \frac{\kappa}{6}
		\end{split}
	\end{align}
	The last inequality holds with probability at least $1 - 4(p-k)k \exp(- \frac{m \kappa}{768k^2 (1 + 4\sigma^2) \max_{l} \Sigma_ll^2 }) - 2\exp(- \frac{c\kappa \alpha_1^4 m}{24 k} + k) - 2 \exp( - \frac{c\alpha_1^2 m}{4} + k)$ by choosing $\delta = \sqrt{\frac{\kappa}{6}}$ in the first and third inequality of equation \eqref{eq:helper mutual incoherence}. By combining all the above results, we prove Lemma \ref{lem:sample mutual incoherence condition}. 
\end{proof}

\section{Effects of outlier proportion on recovery}
\label{sec:effects of outlier proportion}

Most of the existing literature in outlier-robust regression provide theoretical guarantees for the recovery of regression parameter vector which depend on the proportion ($\mu$) of outlier samples~\citep{liu2020high,chen2013robust}. \cite{liu2020high} provides a bound on the norm-error ($\| \widehat{\theta} - \theta^* \|_2$) in the order of $\calO(\max(\frac{1}{\sqrt{n}}, \mu ))$ in Corollary 3.1 and the note below Corollary 3.1, where $\mu=\epsilon$ in their notation. Similarly, \cite{chen2013robust} provides a bound of $\calO(\frac{1}{\sqrt{n}} + \mu )$ in Theorem 3 where $\mu=n_1/n$ and $n_1$ is the number of outliers in their notation. The norm-error for these methods asymptotically go to zero if the proportion of outliers go to zero asymptotically.  This restricts the applicability of these methods to a small regime (left to the vertical black line in Figure~\ref{fig:recwithproportion}). For example, for $\| \widehat{\theta} - \theta^* \|_2$ to have a reasonable rate of $\calO(\frac{1}{\sqrt{n}})$, \citep{liu2020high,chen2013robust} require the outlier proportion $\mu$ to be $\calO(\frac{1}{\sqrt{n}})$. Asymptotically, the proportion of outliers ($\frac{1}{\sqrt{n}}$) approaches zero as the total number of samples ($n$) approaches infinity, which makes it less interesting. We found empirically that even the standard lasso performs decently well in this regime. 

\begin{figure*}[!ht]
	\centering
	\begin{subfigure}{.5\textwidth}
		\centering
		\includegraphics[width=\linewidth]{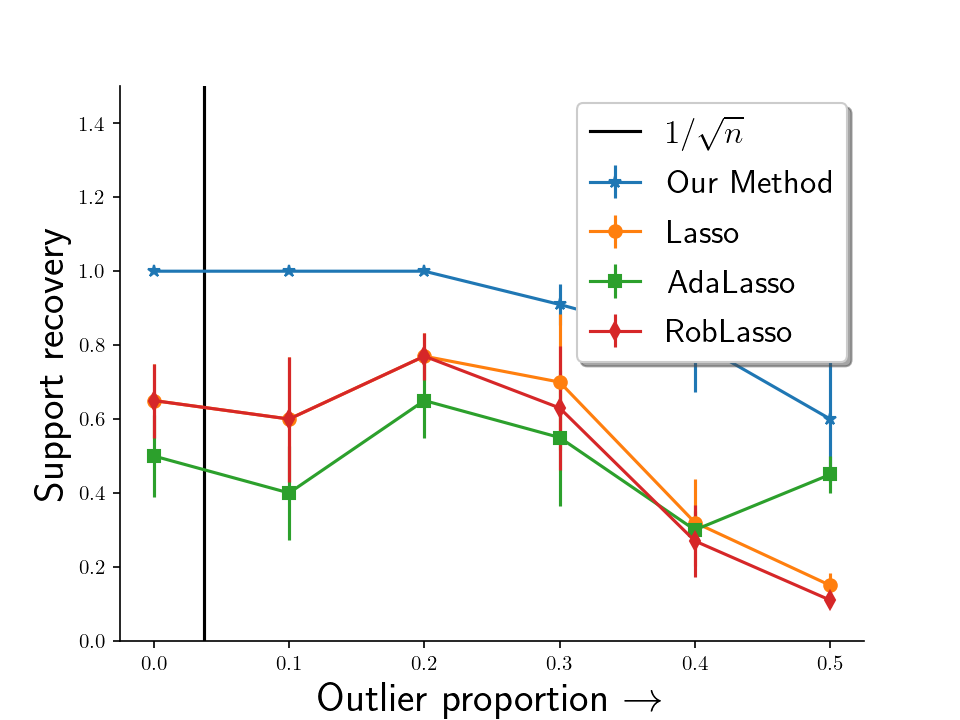}
		\caption{}
		\label{fig:recsupport}
	\end{subfigure}%
	\begin{subfigure}{.5\textwidth}
		\centering
		\includegraphics[width=\linewidth]{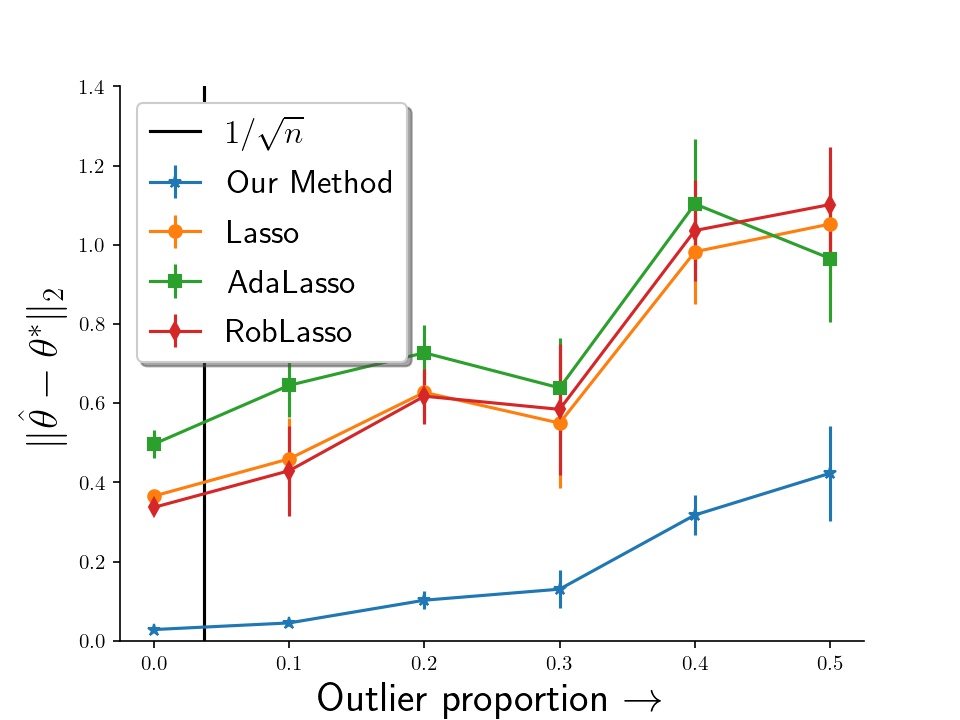}	
		\caption{}
		\label{fig:recerror}
	\end{subfigure}
	\caption{(Left) Support recovery with outlier proportion for $p=30, k=4$. (Right) Norm error for $p=30, k=4$ against outlier proportion. Comparison to standard lasso, AdaLasso~\citep{lambert2011robust} and RobLasso~\citep{chen2013robust} is also shown.}
	\label{fig:recwithproportion}
\end{figure*}

Figure~\ref{fig:recwithproportion} shows that all methods perform reasonably well when number of outliers are small but their performance deteriorates when we increase the number of outliers (the results are averaged across $5$ independent runs.). It can also be seen that our method outperforms all the other methods in all the regimes. Thus, we extend the theoretical guarantees of outlier-robust recovery to a broader and unexplored regime of constant outlier proportions.

\end{document}